\documentclass[preprint,12pt]{elsarticle}
\usepackage{amsmath,amsfonts,amsthm}
\usepackage{algorithmic}
\usepackage{algorithm}
\usepackage{array}
\usepackage[caption=false,font=normalsize,labelfont=sf,textfont=sf]{subfig}
\usepackage{textcomp}
\usepackage{stfloats}
\usepackage[hyphens]{url}
\usepackage{verbatim}
\usepackage{graphicx}
\usepackage{cite}
\usepackage{balance}

\usepackage[scr=esstix,cal=boondox]{mathalfa}
\usepackage[hidelinks]{hyperref}
\usepackage{xcolor}
\usepackage{multirow}
\usepackage{colortbl} 
\usepackage{hhline}

\newtheorem{EO}{Experimental Observation}
\newtheorem{matrice}{Confusion Matrix}
\newtheorem{theorem}{Theorem}
\newtheorem{corollary}{Corollary}
\theoremstyle{definition} 
\newtheorem{example}{Example}

\newcommand{\PCM}{\mbox{{\rm p}-$\mathcal{C}$\hspace{-0.45mm}$\mathscr{o}$\hspace{-0.1mm}$\mathscr{n}$\hspace{-0.1mm}$\mathscr{v}$}}

\usepackage{amssymb}
\journal{\raisebox{-1mm}{\hspace{-4.6mm}\color{white}{\huge$\blacksquare$}}}

\begin{document}

\begin{frontmatter}



\title{Prime Convolutional Model: Breaking the Ground for Theoretical Explainability\tnoteref{t1}}
\tnotetext[t1]{This research project has been supported by a Marie Sk\l odowska-Curie Innovative Training Network Fellowship of the European Commission’s Horizon 2020 Programme under contract number 955901 CISC.}

\author[1]{Francesco Panelli}
\ead{francesco.panelli@mathema.com}
\author[2]{Doaa Almhaithawi\corref{cor1}}
\ead{doaa.almhaithawi@polito.it}
\author[2]{Tania Cerquitelli}
\ead{tania.cerquitelli@polito.it}
\author[3]{Alessandro Bellini}
\ead{abel@mathema.com}

\cortext[cor1]{Corresponding author}

\affiliation[1]{organization={Independent Researcher},
city={Firenze},
country={Italy}
}
\affiliation[2]{organization={Politecnico di Torino, Department of Control and Computer Engineering},
%addressline={Radarweg 29},
%postcode={1043 NX},
city={Torino},
country={Italy}}
\affiliation[3]{organization={Mathema srl},
%addressline={Radarweg 29},
%postcode={1043 NX},
city={Firenze},
country={Italy}}

\begin{abstract}
In this paper, we propose a new theoretical approach to Explainable AI. Following the Scientific Method, this approach consists in formulating on the basis of empirical evidence, a mathematical model to explain and predict the behaviors of Neural Networks. We apply the method to a case study created in a controlled environment, which we call Prime Convolutional Model (\PCM\ for short). \PCM\ operates on a dataset consisting of the first one million natural numbers and is trained to identify the congruence classes modulo a given integer $m$. Its architecture uses a convolutional-type neural network that contextually processes a sequence of $B$ consecutive numbers to each input. We take an empirical approach and exploit \PCM\ to identify the congruence classes of numbers in a validation set using different values for $m$ and $B$. The results show that the different behaviors of \PCM\ (i.e., whether it can perform the task or not) can be modeled mathematically in terms of $m$ and $B$. The inferred mathematical model reveals interesting patterns able to explain when and why \PCM\ succeeds in performing task and, if not, which error pattern it follows.
\end{abstract}

%\begin{center}
%\includegraphics[width=10cm]{Graphical_Abstract.jpg}
%\end{center}

%\item A new theoretical approach to explain convolutional neural networks mainly based on patterns.
%\item An inferred mathematical model to explain when and why a convolutional-type neural network identifies the congruence classes modulo a given integer $m$.
%\item The model implements the multiplicative and additive structures of natural numbers through two complementary blocks of a neural network architecture.


\end{frontmatter}


\section{Introduction}\label{Introduction}
In the last years, Neural Networks have proven to have an extraordinary ability to address and solve a wide range of problems coming from different fields of human experience and knowledge; the spectrum of their applicability is extremely broad, ranging from data processing and generation (e.g., images \citep{dong2015image}, \citep{li2022comprehensive}, \citep{ning2020multi}, \citep{Shen2020InterFaceGANIT}; texts \citep{de2021survey}; sounds \citep{briot2020deep} or even multimodality \citep{baltruvsaitis2018multimodal}), image segmentation (cf. \citep{chen2017deeplab}, \citep{minaee2021image}), pattern recognition and detection (cf. \citep{li2022comprehensive}), medical diagnosis (cf. \citep{bakator2018deep}, \citep{soffer2019convolutional}), decision support (cf. \citep{mumali2022artificial}), robotics (cf. \citep{chen2022neural}, \citep{chen2022recurrent}), and much more. The success of Neural Networks has been striking, and their ability to produce new and unexpected results continues to surprise.

Among other recently developed applications, it turns out to be particularly intriguing the possibility of using Artificial Intelligence techniques in addressing mathematical problems. Many recent works have revealed the power of Neural Networks in supporting the study of pure Mathematics either by providing new tools for calculus (cf. \citep{anitescu2019artificial}, \citep{Lample2020Deep}) or also by guiding human intuition in formulating new reliable conjectures that can subsequently be proved (cf. \citep{davies2021advancing}).

On the other hand, it has been observed that mathematical methods can be applied to deepen our knowledge of the behavior of Neural Networks and to process or generate data more efficiently. Various research activities have shown that latent spaces of different types of networks (e.g., GANs \citep{aggarwal2021generative} and VAEs \citep{asperti2021survey}) can be successfully investigated with Riemannian Geometry techniques (cf. \citep{benfenati2023singular1}, \citep{benfenati2023singular2}, \citep{chen2018metrics}, \citep{hauser2017principles}). Besides, many authors have pointed out the existence of relationships between semantic features of the data (e.g., biological attributes in facial images) and geometric properties of the points representing these data in latent spaces (cf. \citep{fetty2020latent}, \citep{giardina2022naive}, \citep{Shen2020InterFaceGANIT}); such relationships can be captured by appropriate Riemannian metrics, which can be approximated, for example, by applying techniques of non-linear manifold statistics (cf. \citep{arvanitidis2017latent}).

The main problem in performing this type of analyses is that they aim to relate objectively determined characteristics to semantic features of the data, most of which are subjective. Moreover, in many applications, the data have several different attributes that may be interwoven in different and unexpected ways (as is the case, for example, in image generation when it comes to biological characteristics of faces such as the shape of the nose, the color of the hair, cf. \citep{giardina2022naive}, \citep{Shen2020InterFaceGANIT}). All these facts make the analysis of the behavior of the network rather complex, and therefore it proves challenging to hypothesize precisely what factors are associated with the correct functioning of the network and why.

In this paper, we propose a new approach to investigate the behaviors of Neural Networks. Following the Scientific Method, this approach aims to provide, on the basis of empirical evidence, a mathematical model that explains the internal functioning of Neural Networks and highlights the main relationships between inputs and outputs in a clear and precise way.

We shall apply our approach to a well-definable and modelable case study built within a controlled environment. To overcome the problem of relating subjective data characteristics to objectively determined mathematical quantities, we propose to use a dataset that is exclusively endowed with objective characteristics, namely a mathematical dataset. The main advantage is that the features and the relationships between them can be controlled externally by strict mathematical rules, which allow the formulation of precise hypotheses to explain the behaviors of the network in rigorous mathematical terms.

For simplicity, our choice regarding the dataset then falls on a finite set of natural numbers (i.e., the first one million numbers (from $0$ to $999\,999$)). The task that we choose for our experiments is \textit{identifying the congruence classes modulo a given integer $m$}.

To accomplish this task, we develop a convolutional network model that depends on several hyperparameters. The model's performances as a function of different values of the hyperparameters have revealed interesting behaviors of the networks, which we explain in terms of the arithmetic properties of numbers.

The model implements two complementary aspects of the algebraic nature of $\mathbb{N}$: the multiplicative and the additive structures.

The multiplicative structure is encoded in the representation of the input data, i.e., in the method by which we convert numbers into vectors to make them accessible to the network. This method is called {\em prime grid vector representation}, and its strength lies, as the name suggests, in the use of prime factorization of integers to represent numbers as vectors.

The additive structure is encoded in the network architecture. Specifically, the network is a standard convolutional architecture that processes each number $n$ in the input set in a context-dependent manner as a sequence of $B$ consecutive numbers starting at $n$ rather than as a single entity. This way of processing the data is the crucial point that allows the network to obtain, at least locally, information about the additive structure of $\mathbb{N}$. A measure of the extent of this locality is the length $B$ of the sequence of numbers processed with each input, and as might be expected, it plays a fundamental role in explaining the various behaviors of the network.

These facts have led us to name our model {\em Prime Convolutional Model}, \PCM\ for short.

The extensive experimental results reveal the existence of precise relationships between input and output and accurately explain the behaviors of the network (i.e., it is easy to understand when \PCM\ can identify the congruence classes modulo $m$ and, if not, why). Based on the empirical observations, we draw the theoretical consequences for describing and explaining the behaviors of \PCM\ through a mathematical formulation that highlights interesting patterns able to explain when and why \PCM\ succeeds in performing task and, if not, what error pattern it follows. Generalizing the results for \PCM\ to other types of neural models, with other types of datasets and networks, would pave the way for a new, theoretical concept of \textquotedblleft explainability\textquotedblright\ in artificial intelligence.

% The extensive experimental results reveal the existence of precise relationships between input and output and accurately explain the behaviors of the network (i.e., it is easy to understand when \PCM\ can identify the congruence classes modulo $m$ and, if not, why). Based on the empirical observations, we draw the theoretical consequences for describing and explaining the behaviors of \PCM\ by a mathematical formulation. Generalizing the results for \PCM\ to other types of neural models, with other types of datasets and networks, would pave the way for a new, theoretical concept of \textquotedblleft explainability\textquotedblright\ in artificial intelligence.

The paper is organized as follows. In Section~\ref{Literature_Review} we give a review of the main literature related to our work. In Section~\ref{Methodology}, we describe our methodology and present the Prime Convolutional Model in detail: the representation of the input data, the architecture, the task we use for training, and the main evaluation measures we use to analyze the results. In  Section~\ref{Experimental_Results} we discuss the experimental results. In Section~\ref{Theoretical_consequences}, we show how the empirical observations can be used to deduce some theoretical consequences on the behavior of \PCM. Finally, in Section~\ref{Conclusion} we draw some conclusions about our work and show how it may hopefully open up new possibilities for the future.

\section{Literature Review}\label{Literature_Review}
In this Section, we discuss the current state of research on neural networks and explainable artificial intelligence in terms of theoretical perspective and existing examples of collaboration between mathematics and machine learning. All the fields discussed below are and continue to be very active; the various topics covered arouse the interest of the scientific community due to their applicability to many concrete problems of daily life.

\emph{Neural Networks}. The contributions of Artificial Neural Networks in real-world applications are remarkable, and the variety of these applications has no limits \citep{abiodun2018state}. Convolutional Neural Networks (CNNs) are still considered the leading architectures along with Generative Adversarial Networks (GANs) and Variational Autoencoders (VAEs). While the latter two focus on generation and feature extraction, the former are mainly used for pattern recognition (cf. \citep{ning2020multi}). Although CNNs were first introduced to address image-driven pattern recognition tasks (cf. \citep{o2015introduction}), their applications also include computer vision (cf. Segmentation \citep{chen2017deeplab} and \citep{minaee2021image}, resolution augmentation \citep{dong2015image}), natural language processing (cf. \citep{li2021survey}), biomedical diagnosis, anomaly and fault detection (cf. \citep{kiranyaz20211d}), but are not limited to these.

\emph{Explainable Artificial Intelligence (XAI)}. Despite the success of AI models, it is still difficult to trust their results because it is not easy to understand how they are obtained. In fact, these models are often compared to \textquotedblleft black boxes\textquotedblright, whose interpretability can be very challenging, for example when trying to explain a wrong predictive result (cf. \citep{goebel2018explainable}). Several techniques have been developed to provide transparency and interpretability to the choices made by neural networks (cf. \citep{dovsilovic2018explainable}). These techniques mainly follow an \textquotedblleft a posteriori\textquotedblright\ approach, which extracts information from already learned models without precisely understanding the inner decision mechanisms. In \citep{angelov2020towards}, a prototype-based approach for explainable deep neural networks is presented; in \citep{rieger2019aggregating}, an aggregated explainability approach is presented and evaluated; in \citep{saadallah2021explainable}, an explainability approach is presented to support the selection of online CNN-based models using saliency maps for time series forecasting; \citep{samek2016evaluating} explores and evaluates heat maps for assessing neural network performances. Other works seek to replace \textquotedblleft black box\textquotedblright\ deep neural networks with a process that first aims to learn a set of interpretable concepts and then uses these concepts to perform a classification task (cf. \citep{koh2020concept} and its improved version \citep{espinosa2022concept}). Some recent attempts try to look at explainability from a more theoretical point of view. For example, \citep{jia2021studying} examines the relationship between the accuracy of the model and the quality of the explainability; \citep{debbi2021causal} uses techniques from the theory of causality on Convolutional Neural Networks to measure the relevance of the local features of an input image in the network's decision process; \citep{wang2023generalized} implements a new class of explanations, the so-called \textquotedblleft deliberative explanations\textquotedblright, which visualize the regions of an image that the network considers ambiguous for the proposed classification task; \citep{slack2021reliable} presents a Bayesian framework for generating local explanations along with the associated uncertainty; \citep{luo2022learning} develops a method for dynamically learning differential relations from input data to explain the time-evolving dynamics of Time Series Models.

\emph{AI techniques supporting abstract Mathematics}. The use of neural networks to support Mathematics is manifold. Several techniques have been developed to approximate the solutions of differential equations. For example, in \citep{anitescu2019artificial}, second order boundary value problems are approached with an adaptive collocation strategy; in \citep{raissi2019physics}, neural networks are trained to solve a supervised learning task while respecting a given physical law expressed by a non-linear PDE. Another active field of research concerns the possibility of employing artificial intelligence to deal with complex symbolic expressions. In \citep{allamanis2017learning}, a new architecture is proposed to learn the semantic representation of a symbolic expression; in \citep{Lample2020Deep} it is shown how neural networks can successfully deal with challenging tasks involving symbolic calculus like the integration of functions or the solution of ordinary differential equations. Learning algorithms is another problem that has been widely investigated, from the introduction of the so-called \textquotedblleft Neural Turing Machines\textquotedblright\ in \citep{graves2014neural} to the development of other architectures such as the Neural GPUs in \citep{kaiser2015neural} and the Neural ALUs in \citep{trask2018neural} that learn algorithms to perform elementary arithmetic or logic operations. Neural Networks have also been successfully used to support automatic theorem provers as in \citep{loos2017deep} and, more recently, in \citep{davies2021advancing}, to guide intuition in formulating new reliable conjectures to be subsequently proved by humans.

\emph{Mathematics supporting Neural Networks}. In recent years, many attempts have been made to build a mathematical theory of Neural Networks within a precise mathematical framework. Most of these attempts try to describe neural networks using techniques from Riemannian (cf. \citep{hauser2017principles}) and Pseudo-Riemannian Geometry (cf. \citep{benfenati2023singular1}, several applications of which are presented in \citep{benfenati2023singular2}), the goal being always that of finding, within latent spaces, the best metric that fits the data (i.e. a metric with respect to which semantically related data are represented by close points). Several results have been obtained in this direction. For example, for Deep Generative Models, in \citep{arvanitidis2017latent} it is shown how the non-linearity of latent spaces can be characterized by a stochastic metric; in \citep{kuhnel2018latent} the latent space metric, which is approximated by a different neural network, is used to implement various non-linear manifold statistics techniques; \citep{chen2018metrics} and \citep{shao2018riemannian} develop algorithms to implement geometric objects such as geodesics and parallel translations, and show how these concepts can be used to highlight semantic features of the data or in concrete applications such as robot movements. 
The semisupervised distance metric learning problem has also been successfully addressed in \citep{ying2017manifold} using Differential Geometry techniques. 
The geometry of latent spaces has been extensively investigated for several types of Generative Adversarial Networks, especially in the field of image generation. The problem consists in finding ways to embed a real image into the latent space, and then in moving towards close points in order to modify it (cf. \citep{abdal2019image2stylegan}); the role of geometry then comes into play to gain control over the modifications (cf. \citep{fetty2020latent}). The interplays between geometry and the features of the data are studied in \citep{giardina2022naive} and \citep{Shen2020InterFaceGANIT} for different GAN architectures (namely StyleGAN2 and InterFaceGAN respectively).

In this paper, we propose a new method by which Mathematics can support the study of Neural Networks. Our approach, contextualized in the framework of Explainable AI, consists in developing, on the basis of empirical evidence a mathematical model to predict and explain the behaviors of the networks. We shall apply the method to a specific case study created in a controlled environment, which we call {\em Prime Convolutional Model}, \PCM\ for short. \PCM\ works on a dataset consisting of the first one million natural numbers (from $0$ to $999\,999$), and is trained to identify the congruence classes modulo a given integer $m$.

These types of dataset and task were first considered in \citep{almhaithawi2023construction}, where the performances of two different network architectures and several input data vector representations were compared. The conclusion was that, for $m\in\{2,3,\dots,10\}$, the model can solve the proposed task precisely when an architecture of convolutional type and the so-called \textquotedblleft Prime Grid vector representation\textquotedblright\ are used.
The scope of this work is indeed completely different from that of \citep{almhaithawi2023construction}; it deals with a different research goal and paves the way for theoretically explainable AI. The only common aspects are the task, i.e., identifying the congruence classes modulo a given integer $m$ in a dataset consisting of the first one million natural numbers, and the model architecture.

\section{Methodology}\label{Methodology}
Our new theoretical approach to explainable AI follows the Scientific Method. First, we build, within a controlled environment, a case study consisting of a suitably chosen neural network model. Then, based on the results obtained by performing several experiments, we formulate some general rules about the model's behaviors in precise mathematical terms. Finally, we elaborate these rules mathematically to infer further behaviors, which are then verified empirically.

In this Section we present the neural network model that we use to perform our analysis. This model will be applied to a specific mathematical problem, namely \textit{identifying the congruence classes modulo $a$ given integer $m$} (cf. Subsection \ref{Task}), in a dataset consisting of the first one million natural numbers (from $0$ to $999\,999$). The architecture, which is of convolutional type, depends on several hyperparameters (cf. Subsection \ref{Hyperparameters}); the performances of the model as a function of these hyperparameters reveal interesting behaviors, which we explain in terms of the arithmetic properties of the numbers. More specifically, the model captures two complementary algebraic properties of the dataset: the multiplicative and the additive structures. These goals are achieved, respectively, by the representation of the input data (cf. Subsection \ref{IDR}), which exploits the prime factorization of the integers, and by the convolutional architecture of the involved network (cf. Subsection \ref{Arch}). For this reason, we call our model \textquotedblleft Prime Convolutional Model\textquotedblright, \PCM\ for short.

We provide an implementation of \PCM\ in a notebook available on the Code Ocean website.

\subsection{Input Data Representation}\label{IDR}
The vector representation that we choose for the input data is the so-called {\em prime grid vector representation}; it consists of an implementation of the arithmetic concept of \textquotedblleft Prime Grid\textquotedblright, first introduced in \citep{kolossvary2022distance}. This concept is closely related to the prime factor decomposition of natural numbers, and this explains the ability of its implemented version, the prime grid vector representation, to straightforwardly encode the multiplicative structure of $\mathbb{N}$.

If $\mathcal{P}=\{p_i\,|\,i=1,2,\dots\}$ is the sequence of all prime numbers in ascending order, the {\em Prime Grid} is the set $\mathbb{N}^{\mathcal{P}}$ of all the infinite sequences of natural numbers indexed on  $\mathcal{P}$. An element $n\in\mathbb{N}$ can then be represented in $\mathbb{N}^\mathcal{P}$ through its {\em prime signature}, that is, through the unique sequence $(\ell_1,\ell_2,\ell_3\dots) \in \mathbb{N}^\mathcal{P}$ such that
\begin{equation}\label{pfd}
n=\prod_{i=1}^\infty p_i^{\ell_i}.
\end{equation}
Since in (\ref{pfd}) only finitely many exponents are $\neq 0$, all the prime signatures of natural numbers are necessarily eventually zero (that is, all of their entries are zero from some point on);
this observation readily leads to a manageable implementation of the Prime Grid by truncation. For a dataset consisting of the first one million natural numbers, it is clear that the optimal level $N$ for truncating the prime signature is equal to the number of primes $<1\,000\,000$, that is $N=78\,498$; this avoids any loss of information since all ignored entries are $0$. However, in order to reduce the computational complexity of the implementation, we shall perform the truncation at level $N=5\,000$, so that the dataset will be restricted to those numbers
that can be factorized using only the first 5~000 primes. This reduction turns out to be not too invasive: in fact, it preserves 785~095 numbers, corresponding to about 79\% of the original dataset.

\begin{example}
\rm The prime signature of $20=2^2\cdot 5^1$ is the {\em infinite} sequence $(2,0,1,0,0,\dots)$; the prime grid vector representation of $20$ is therefore the $5000$-vector $(2,0,1,0,0,\dots,0)$. Similarly, the prime signature of $126=2^1\cdot 3^2\cdot 7^1$ is $(1,2,0,1,0,0,\dots)$, while its prime grid vector representation is the $5000$-vector $(1,2,0,1,0,0,\dots,0)$.
\end{example}

\subsection{Architecture}\label{Arch}
The architecture of the model is a standard convolutional architecture (cf. \citep{o2015introduction}, \citep{li2021survey}, \citep{ajit2020review}). The strength of this approach consists in processing the data in sequences; this is realized by organizing each element $n$ of the dataset into a $B\times N$ matrix whose rows contain the prime grid vector representations of a sequence of $B$ consecutive numbers starting at $n$. 

\begin{figure*}[!t]
\centering
\includegraphics[width=\textwidth]{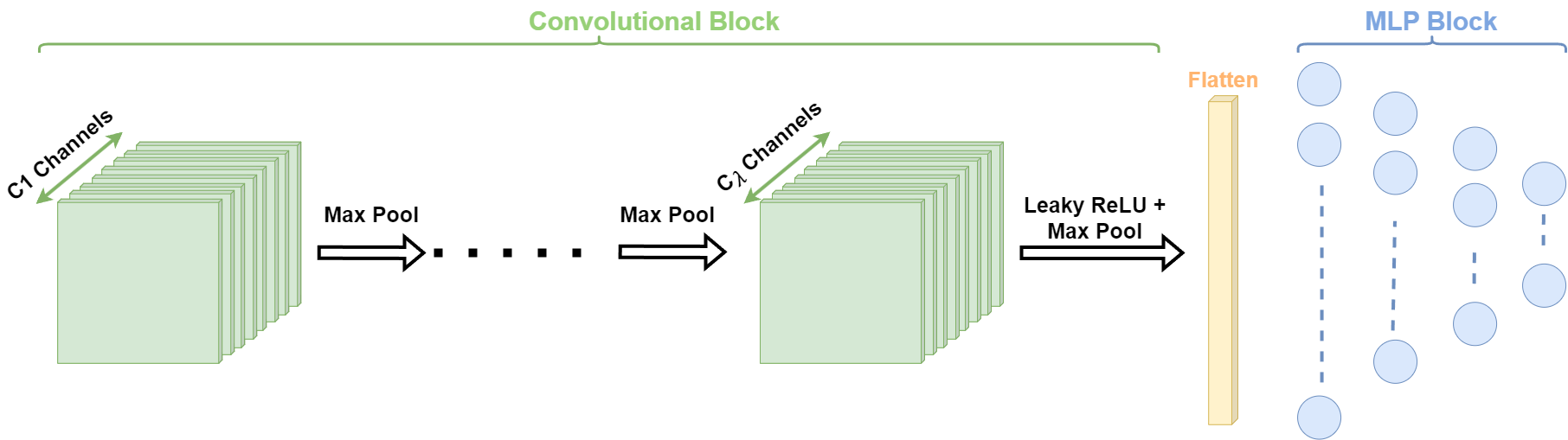}
\caption{Standard CNN architecture.}
\label{fig_OriginalArc}
\end{figure*}

The model consists of a convolutional block and multilayer perceptron block (cf. Figure \ref{fig_OriginalArc}). The former contains $\lambda$ convolutional layers, each followed by max-pooling with a stride of two; for $i=1,\dots,\lambda$ layer $i$ has $C_i$ channels, each acted on by a kernel of size $k_i\times k_i$. Layer $\lambda$ is also followed by an application of the Leaky ReLU function. A flattened layer then reorganizes the data into one-dimensional arrays to be processed by the multilayer perceptron block, which consists of four fully connected linear layers, each followed (except for the last) by the Leaky ReLU function. The dimensions of the fully connected linear layers follow the scheme:
\begin{equation}\label{scheme}
F\rightarrow 1000 \rightarrow 100 \rightarrow 10\rightarrow m,
\end{equation}
where $F$ is the dimension of the flatten layer and $m$ is the dimension of the output space.

\subsection{The Task}\label{Task}
Our \PCM\ is trained to identify the congruence classes modulo $a$ given natural number $m\geq 2$. Recall from elementary arithmetic that two integers $z_1$, $z_2$ are said to be congruent modulo $m$, in symbols 
$$
z_1\equiv z_2\mod m,
$$
if $m$ divides $z_1-z_2$ or, equivalently, if $z_1$ and $z_2$ produce the same remainder in the division by $m$. Given $z\in\mathbb{Z}$, the congruence class of $z$ modulo $m$ is the set of all integers that are congruent to $z$ modulo $m$, and it can be explicitly described as the set
$$
[z]_m=\{z+\xi m\,|\,\xi\in\mathbb{Z}\}.
$$
There exist precisely $m$ distinct congruence classes modulo $m$, namely $[0]_m,[1]_m,\dots$, $[m-1]_m$; we therefore train the model so to assign to each number $n$ in the input space the correct label $L\in\{0,1,\dots,m-1\}$ that corresponds to the congruence class of $n$ modulo $m$.

\begin{example}
There are precisely seven congruence classes modulo $m=7$, namely:
$$
\begin{array}{l}
[0]_7=\{\dots, -21, -14, -7,0, 7, 14, 21,\dots\};    \\  {[1]_7=\{\dots,-20, -13, -6, 1, 15, 22,\dots\};}\\
{[2]_7=\{\dots,-19, -12, -5, 2, 16, 23,\dots\};}   \\ {[3]_7=\{\dots,-18, -11, -4, 3, 17, 24,\dots\};}\\
{[4]_7=\{\dots,-17, -10, -3, 4, 18, 25,\dots\};} \\ {[5]_7=\{\dots,-16, -9, -2, 5, 19, 26,\dots\};}\\
{[6]_7=\{\dots,-15, -8, -1, 6, 20, 27,\dots\}.}
\end{array}
$$
\end{example}

The training occurs in batches: we use $r$ disjoint batches, each consisting of $s$ distinct numbers randomly chosen from the input set. In this way the size of the training set is $r\cdot s$ numbers. 

\begin{example}
As we shall see in \ref{App_Training_Hyp}, in our applications we use $r=400$ batches, each containing $s=256$ distinct numbers; therefore our training set consists of $r\cdot s=400\cdot 256= 102\,400$ numbers. 
\end{example}

The optimization step is carried out by the Adam optimizer, and the function to be optimized is the Cross-Entropy loss function (cf. \citep{NEURIPS2019_9015} for the implementations). The training lasts $t$ epochs, during each of which the model is applied to all batches in sequence: after application to the $i$-th batch, the weights are updated accordingly, and then the new model is applied to the $(i+1)$-th batch.

\subsection{Evaluation Measures}
After training, the model is validated on a batch of $512$ distinct numbers randomly chosen outside of the training set.
In order to interpret the results of the validation and analyze the corresponding performance of the model, we rely on two classical evaluation measures: accuracy, that is, the ratio between the number of labels that the model has correctly assigned and the size of the validation set; and confusion matrix, that is, the $m\times m$ matrix $C=(c_{ij})$ whose \mbox{$(i,j)$-th} entry $c_{ij}$ denotes the number of elements in the validation set that have label $i$ and to which the model has assigned label~$j$. 

\subsection{Impact of Hyperparameters on \PCM\ development}\label{Hyperparameters}
As shown in the previous Subsections, \PCM\ depends on several hyperparameters; the choice of the optimal values for each of them requires a rather technical and tedious analysis that we shall present in the Appendix so as not to interrupt the discussion at this level. In this Subsection we shall then only categorize these hyperparameters into three distinct classes depending on their role in the architecture of \PCM, and present the upshots of the tuning analysis listing, for each hyperparameter, the values that we choose to perform the experiments.

The classes of hyperparameters that we distinguish are the following:

\begin{itemize}
    \item {\bf Convolutional hyperparameters}. These parameters are involved in the convolutional block of the architecture of \PCM; they are the number $\lambda$ of convolutional layers and the numbers $C_1,\dots,C_\lambda$ of channels for each layer. As we shall see in \ref{App_Conv_Hyp}, these hyperparameters only affect the computational complexity of the model without modifying too much its behaviors; for this reason the values chosen for them are as small as possible: $\lambda=1$, and $C_1=C=4$.
    \item {\bf Training hyperparameters}. These parameters are involved in the training process of the network; they are the number $r$ of batches, the size $s$ of every batch, and the number $t$ of epochs. As we shall see in \ref{App_Training_Hyp}, these hyperparameters affect mainly the model building time. We shall use for them relatively low values at which the typical behaviors of \PCM\ are already visible: $r=400$, $s=256$ and $t=10$ (however we do not present systematically the results of the 10th epoch, but rather those of the epoch at which the best accuracy is attained).
    \item {\bf Locality hyperparameters}. These parameters are closely related to the network's local understanding of the additive structure in $\mathbb{N}$; they are the length $B$ of the sequence of numbers contextually processed with each input and the sizes $k_1,\dots,k_\lambda$ of the kernels acting on the channels. The experiments that we shall present in Section \ref{Experimental_Results} reveal how different performances and behaviors of the model in correspondence of different values of $B$ can shed light on the inner mechanisms that govern the functioning of our network. For this reason, we shall not pick a unique value for the locality hyperparameters, but rather a set of significant values that illustrate the totality of behaviors that can occur. For reasons to be discussed in depth in \ref{App_Loc_Hyp}, we take $B\in\{8,16,24\}$. Furthermore, if $B=8$ the size $k_1=k$ of the kernel is set to $k=7$; if $B=16$ or $24$ we compare the outcomes of \PCM\ with $k=7$ and $k=15$.
\end{itemize}

The final architecture of \PCM\ is represented in Figure \ref{PCM_final_arch}.

\begin{figure*}[!t]
\centering
\includegraphics[width=0.7\textwidth]{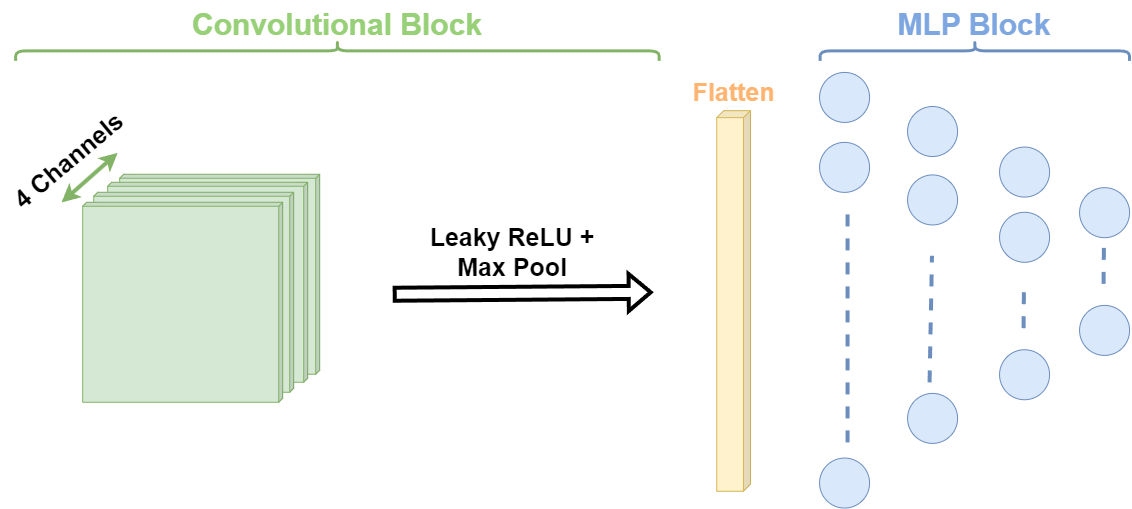}
\caption{\PCM\ final architecture.}
\label{PCM_final_arch}
\end{figure*}

\section{Theoretical Observations and Experimental Results}\label{Experimental_Results}
In order to study \PCM, we conduct several experiments to classify the congruence classes modulo an integer $m\in\{2,3,\dots,30\}$. A careful analysis of the results shows that \PCM\ follows general behavioral schemes; in this section we present these schemes through clear statements, that we call \textquotedblleft Experimental Observations\textquotedblright, and provide empirical support for them.

\begin{EO}\label{recognized_classes}
\PCM\ always identifies class $[0]_m$ and the last $B$ classes.
\end{EO}

It is not really surprising that the network may recognize class $[0]_m$: indeed, this information is encoded within the Prime Grid vector representation. It is, however, actually unexpected that the network may recognize classes $[-1]_m$, $[-2]_m,\dots$ $[-B]_m$ for any choice of $m$ and $B$. 

Note that Experimental Observation \ref{recognized_classes}, in particular, implies that the task of identifying the congruence classes modulo $m$ is fully solved by \PCM\ if
\begin{equation}\label{mB2}
    m\leq B+2.
\end{equation}
When \eqref{mB2} is not satisfied, then the model does not necessarily fail in identifying all the congruence classes. All the performed experiments reveal in fact the existence of a close relationship between the solvability of the task of identification of the congruence classes modulo $m$ and the prime factor decomposition of $m$:

\begin{EO}\label{PCM_functioning}
Let $m\geq 2$ be an integer, and let $m=p_1^{\ell_1}p_2^{\ell_2}\cdots p_\alpha^{\ell_\alpha}$ be the prime factor decomposition of $m$. \PCM\ identifies the congruence classes modulo $m$ if and only if it identifies the congruence classes modulo $p_i^{\ell_i}$ for all $i=1,\dots,\alpha$.
%that is, if and only if $p_i^{\ell_i}\leq B+2$ for all $i=1,\dots,\alpha$.
\end{EO}

In order to provide empirical evidence to Experimental Observations \ref{recognized_classes} and \ref{PCM_functioning} we proceed in three steps, distinguishing the cases where the modulo, $m$, is a {\em prime}, a {\em power of a prime}, and finally a {\em splitting number} (i.e. its prime factor decomposition involves at least two distinct primes). 

\subsection{Prime Moduli}\label{PatternA}
When $m$ is a prime number $p$, all the experiments lead to the following:

\begin{EO}\label{EO_prime_moduli}
If $m=p$ is a prime number, then \PCM\ identifies the congruence classes modulo $p$ if and only if $p\leq B+2$. If $p>B+2$ then the classes $[0]_p$, $[-1]_p,\dots$, $[-B]_p$ are correctly identified, while the classes $[1]_p,\dots,$ \mbox{$[m-B-1]_p$} (which are at least $2$) are randomly mixed and confused with one another by \PCM.   
\end{EO}

We remark here that not only the latter statement perfectly agrees with Experimental Observations \ref{recognized_classes} and \ref{PCM_functioning} in the particular case of a prime modulo, but also it provides a clear description of the error pattern followed by the model in all of those cases that it cannot treat correctly. 

In order to verify Experimental Observation \ref{EO_prime_moduli} we exploit \PCM\ to identify the congruence classes modulo all prime numbers $m$ in the range 2--30, that is, $m\in\{2,3,5,7,11,13,17,19,23,29\}$, with $B\in\{8,16,24\}$ and $k\in\{7,15\}$. The accuracies attained are listed in Table \ref{table_pattern_A}.

\begin{table}
\begin{center}
\caption{Accuracies attained by \PCM\ when $m\in\{2,3,\dots,30\}$ is prime, $B\in\{8,16,24\}$ and $k\in\{7,15\}$.}
\label{table_pattern_A}
\begin{tabular}{c|c|cc|cc|} 
\hhline{~-----}
\multicolumn{1}{c|}{} & \multicolumn{1}{c|}{\cellcolor{gray!30}$\boldsymbol{B=8}$} &  \multicolumn{2}{c|}{\cellcolor{gray!30}$\boldsymbol{B=16}$} & \multicolumn{2}{c|}{\cellcolor{gray!30}$\boldsymbol{B=24}$}  \\
\hhline{------}
\multicolumn{1}{|c|}{\cellcolor{gray!30}$\boldsymbol{m}$} & $\cellcolor{gray!30}\boldsymbol{k=7}$ & $\cellcolor{gray!30}\boldsymbol{k=7}$ & $\cellcolor{gray!30}\boldsymbol{k=15}$ & $\cellcolor{gray!30}\boldsymbol{k=7}$ & $\cellcolor{gray!30}\boldsymbol{k=15}$\\
\hline
\multicolumn{1}{|c|}{\cellcolor{gray!30}\bf 2} & 1.00 & 1.00 & 1.00 & 1.00 & 1.00\\
\hline
\multicolumn{1}{|c|}{\cellcolor{gray!30}\bf 3} & 1.00 & 1.00 & 1.00 & 1.00 & 1.00\\
\hline
\multicolumn{1}{|c|}{\cellcolor{gray!30}\bf 5} & 1.00 & 1.00 & 1.00 & 1.00 & 1.00\\
\hline
\multicolumn{1}{|c|}{\cellcolor{gray!30}\bf 7} & 1.00 & 1.00 & 1.00 & 1.00 & 1.00\\
\hline
\multicolumn{1}{|c|}{\cellcolor{gray!30}\bf 11} & 0.92 & 1.00 & 1.00 & 1.00 & 1.00\\
\hline
\multicolumn{1}{|c|}{\cellcolor{gray!30}\bf 13} & 0.76 & 1.00 & 1.00 & 1.00 & 1.00\\
\hline
\multicolumn{1}{|c|}{\cellcolor{gray!30}\bf 17} & 0.59 & 1.00 & 1.00 & 1.00 & 1.00\\
\hline
\multicolumn{1}{|c|}{\cellcolor{gray!30}\bf 19} & 0.54 & 0.94 & 0.95 & 0.99 & 1.00\\
\hline
\multicolumn{1}{|c|}{\cellcolor{gray!30}\bf 23} & 0.38 & 0.70 & 0.77 & {\bf 0.92} & 1.00\\
\hline
\multicolumn{1}{|c|}{\cellcolor{gray!30}\bf 29} & 0.33 & 0.56 & 0.59 & 0.79 & 0.88\\
\hline
\end{tabular}
\end{center}
\end{table}

The results manifestly confirm what has been claimed: with a sequence of length $B=8$ accuracy 1.00 is attained up to $m=7$; with a sequence of length $B=16$ accuracy 1.00 is attained up to $m=17$; with a sequence of length $B=24$ and a kernel size of $k=15$ accuracy 1.00 is attained up to $m=23$. All those situations in which accuracy 1.00 is not reached were predicted by Experimental Observation \ref{EO_prime_moduli}, except that with $m=23$, $B=24$ and $k=7$ (boldfaced in Table \ref{table_pattern_A}). 

To gain a deeper insight into this case, we analyze the corresponding confusion matrix (cf. Confusion Matrix \ref{B24m23k7}). Its form is clear: it is an almost-diagonal matrix with only minor non-zero entries (highlighted in light red) outside of the main diagonal (and mostly close to it). This behavior is explainable by looking at the size of the kernel: if $k$ is too small with respect to $B$, then the model struggles to perform the task. This is not really surprising: if $k/B \ll 1$, the convolutional layer loses its effectiveness.

\begin{figure}[!t]
\centering
\includegraphics[width=7cm]{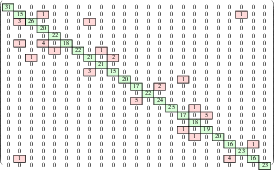}

\vspace{-3mm}
\begin{matrice}\label{B24m23k7}
\centering
$m=23$, $B=24$, $k=7$.
\end{matrice}
\end{figure}

% \begin{figure}[!t]
% \centering
% \includegraphics[width=7cm]{m23-B24-k7.eps}

% \noindent\begin{minipage}{\textwidth}
% \begin{matrice}\label{B24m23k7}
% \centering
% $m=23$, $B=24$, $k=7$.
% \end{matrice}
% \end{minipage}
% \end{figure}

Table \ref{table_pattern_A} shows other cases in which the model does not attain accuracy 1.00: all of these cases were predicted by Experimental Observation \ref{EO_prime_moduli}, however we can gain a better understanding of the error pattern that \PCM\ follows by looking at the confusion matrices. They reveal a common form consisting of three diagonal blocks, differently colored in the examples provided: the blue block corresponds to class $[0]_m$ being correctly identified; the yellow block corresponds to the last $B$ classes being correctly identified; the red block corresponds to the remaining classes being randomly confused with one another. For example, Confusion Matrix \ref{B8m11k7} refers to the case with $m=11$, $B=8$ and $k=7$ and Confusion Matrix \ref{B8m13k7} refers to the case with $m=13$, $B=8$ and $k=7$.

\begin{figure*}[!t]
\centering
\begin{minipage}{0.49\textwidth}
\centering
\includegraphics[width=6cm]{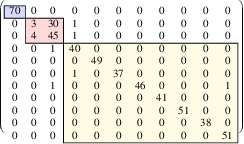}

\noindent\begin{minipage}{7cm}
\begin{matrice}\label{B8m11k7}
\centering
$m=11$, $B=8$, $k=7$.
\end{matrice}
\end{minipage}
\end{minipage}
\begin{minipage}{0.49\textwidth}
\centering
\includegraphics[width=6cm]{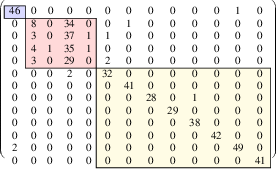}

\noindent\begin{minipage}{\textwidth}
\begin{matrice}\label{B8m13k7}
\centering
$m=13$, $B=8$, $k=7$.
\end{matrice}
\end{minipage}
\end{minipage}
\end{figure*}

% \begin{figure}[!t]
% \centering
% \includegraphics[width=7cm]{m11-B8-k7.eps}

% \vspace{-3mm}
% \begin{matrice}\label{B8m11k7}
% \centering
% $m=11$, $B=8$, $k=7$.
% \end{matrice}
% \end{figure}

% \begin{figure}[!t]
% \centering
% \includegraphics[width=7cm]{m13-B8-k7.eps}

% \vspace{-3mm}
% \begin{matrice}\label{B8m13k7}
% \centering
% $m=13$, $B=8$, $k=7$.
% \end{matrice}
% \end{figure}

It is also interesting to compare the confusion matrices corresponding to cases in which $m$ is fixed and $B$ varies: it is then visible how, in accordance with our general claim, as $B$ increases, the sizes of the yellow block increase as well at the expense of those of the red one (cf. for example, Confusion Matrices \ref{B8m19k7}, \ref{B16m19k15}, in which $m=19$, and Confusion Matrices \ref{B8m29k7}, \ref{B16m29k15}, \ref{B24m29k15} in which $m=29$). 

\begin{figure*}[!t]
\centering
\begin{minipage}{0.49\textwidth}
\centering
\includegraphics[width=5.8cm]{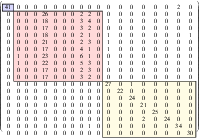}

\noindent\begin{minipage}{7cm}
\begin{matrice}\label{B8m19k7}
\centering
$m=19$, $B=8$, $k=7$.
\end{matrice}
\end{minipage}
\end{minipage}
\begin{minipage}{0.49\textwidth}
\centering
\includegraphics[width=6.6cm]{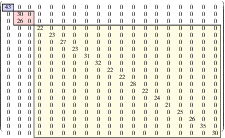}

\noindent\begin{minipage}{\textwidth}
\begin{matrice}\label{B16m19k15}
\centering
$m=19$, $B=16$, $k=15$.
\end{matrice}
\end{minipage}
\end{minipage}
\end{figure*}

\begin{figure}[!t]
\centering
\includegraphics[width=8cm]{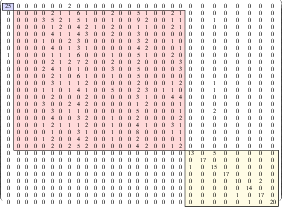}

\vspace{-3mm}
\begin{matrice}\label{B8m29k7}
\centering
$m=29$, $B=8$, $k=7$.
\end{matrice}
\end{figure}

\begin{figure}[!t]
\centering
\includegraphics[width=8.8cm]{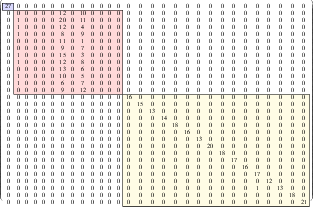}

\vspace{-3mm}
\begin{matrice}\label{B16m29k15}
\centering
$m=29$, $B=16$, $k=15$.
\end{matrice}
\end{figure}

% \begin{figure*}[!t]
% \centering
% \begin{minipage}{0.49\textwidth}
% \centering
% \includegraphics[width=8cm]{m29-B8-k7.eps}

% \noindent\begin{minipage}{7.5cm}
% \begin{matrice}\label{B8m29k7}
% \centering
% $m=29$, $B=8$, $k=7$.
% \end{matrice}
% \end{minipage}
% \end{minipage}
% \begin{minipage}{0.49\textwidth}
% \centering
% \includegraphics[width=8.8cm]{m29-B16-k15.eps}

% \noindent\begin{minipage}{7.5cm}
% \begin{matrice}\label{B16m29k15}
% \centering
% $m=29$, $B=16$, $k=15$.
% \end{matrice}
% \end{minipage}
% \end{minipage}
% \end{figure*}

\begin{figure*}[!t]
\centering
\includegraphics[width=9.2cm]{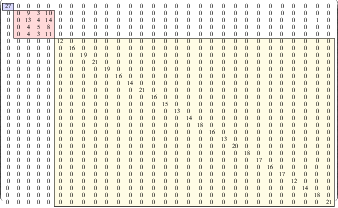}

\noindent\begin{minipage}{\textwidth}
\begin{matrice}\label{B24m29k15}
\centering
$m=29$, $B=24$, $k=15$.
\end{matrice}
\end{minipage}
\end{figure*}

\subsection{Prime Power Moduli}\label{PatternB}
We now consider the situation where $m=p^a$ is a power of some prime number $p$. All the performed experiments show that we can characterize recursively the moduli $m$ whose equivalence classes are correctly identified by \PCM. The case $a=1$ has been already considered in Subsection \ref{PatternA}, therefore we may assume $a\geq 2$.

\begin{EO}\label{EO_prime_power}
Let $m=p^a$ be a power of some prime number $p$, with $a\geq 2$. 
\begin{enumerate}
    \item If $p^a\leq B+2$, then \PCM\ correctly identifies all the congruence classes modulo $p^a$.
    \item If $p^a>B+2$, let $j\in\{0,1,\dots,a-1\}$ be the the largest integer such that the model can correctly identify the congruence classes modulo $p^j$. Then:
    \begin{itemize}
        \item \PCM\ correctly identifies the classes $[0]_{p^a}$, $[-1]_{p^a},\dots$, $[-B]_{p^a}$.
        \item The only difficulty that it exhibits in identifying the classes $[1]_{p^a},\dots$, $[p^a-B-1]_{p^a}$ consists in distinguishing those that belong to the same class modulo $p^j$. 
    \end{itemize}
    In particular, this means that, when $p^a> B+2$, \PCM\ can correctly identify the congruence classes modulo $p^a$ if and only if 
    \begin{equation}\label{identif_cond}
        p^a-B-1\leq p^j.
    \end{equation}
\end{enumerate}
\end{EO}

The latter formulation is in perfect agreement with Experimental Observations \ref{recognized_classes} and \ref{PCM_functioning}; furthermore it precisely describes the error pattern followed by the model in all of those cases in which it is not able to perform the task correctly.  
 
In order to verify Experimental Observation \ref{EO_prime_power}, we exploit \PCM\ to identify the congruence classes modulo all non-trivial prime powers $m$ in the range 2--30, that is, $m\in\{4,8,9,16,25,27\}$, with $B\in\{8,16,24\}$ and $k\in\{7,15\}$. The accuracies attained are listed in Table \ref{table_pattern_B}.

\begin{table}
\begin{center}
\caption{Accuracies attained by \PCM\ when $m\in\{2,3,\dots,30\}$ is a non-trivial power of a prime, $B\in\{8,16,24\}$ and $k\in\{7,15\}$.}
\label{table_pattern_B}
\begin{tabular}{c|c|cc|cc|} 
\hhline{~-----}
\multicolumn{1}{c|}{} & \multicolumn{1}{c|}{\cellcolor{gray!30}$\boldsymbol{B=8}$} &  \multicolumn{2}{c|}{\cellcolor{gray!30}$\boldsymbol{B=16}$} & \multicolumn{2}{c|}{\cellcolor{gray!30}$\boldsymbol{B=24}$}  \\
\hline
\multicolumn{1}{|c|}{\cellcolor{gray!30}$\boldsymbol{m}$} & $\cellcolor{gray!30}\boldsymbol{k=7}$ & $\cellcolor{gray!30}\boldsymbol{k=7}$ & $\cellcolor{gray!30}\boldsymbol{k=15}$ & $\cellcolor{gray!30}\boldsymbol{k=7}$ & $\cellcolor{gray!30}\boldsymbol{k=15}$\\
\hline
\multicolumn{1}{|c|}{\cellcolor{gray!30}\bf 4} & 1.00 & 1.00 & 1.00 & 1.00 & 1.00\\
\hline
\multicolumn{1}{|c|}{\cellcolor{gray!30}\bf 8} & 1.00 & 1.00 & 1.00 & 1.00 & 1.00\\
\hline
\multicolumn{1}{|c|}{\cellcolor{gray!30}\bf 9} & 1.00 & 1.00 & 1.00 & 1.00 & 1.00\\
\hline
\multicolumn{1}{|c|}{\cellcolor{gray!30}\bf 16} & 1.00 & 1.00 & 1.00 & 1.00 & 1.00\\
\hline
\multicolumn{1}{|c|}{\cellcolor{gray!30}\bf 25} & 0.60 & 0.89 & 0.85 & 1.00 & 0.99\\
\hline
\multicolumn{1}{|c|}{\cellcolor{gray!30}\bf 27} & 0.67 & 0.95 & 0.95 & 1.00 & 1.00\\
\hline
\end{tabular}
\end{center}
\end{table}

Once again the results confirm our predictions.
\begin{itemize}
    \item If $B=8$:
    \begin{itemize}
        \item The model correctly identifies the congruence classes modulo 4, 8 and 9.
        \item The model correctly identifies also the congruence classes modulo 16, for in this case $p=2$, $a=4$, $j=3$ and condition \eqref{identif_cond} is satisfied.
        \item The model does not identify the congruence classes modulo 25 and modulo 27: in the former case $p=5$, $a=2$, $j=1$, in the latter case $p=3$, $a=3$, $j=2$ and in neither of them condition \eqref{identif_cond} is satisfied.
    \end{itemize}
    \item If $B=16$:
    \begin{itemize}
        \item The model correctly identifies the congruence classes modulo 4, 8, 9 and 16.
        \item The model does not identify the congruence classes modulo 25 and modulo 27: in the former case $p=5$, $a=2$, $j=1$, in the latter case $p=3$, $a=3$, $j=2$ and in neither of them condition \eqref{identif_cond} is satisfied.
    \end{itemize}
    \item If $B=24$:
    \begin{itemize}
        \item The model correctly identifies the congruence classes modulo 4, 8, 9, 16 and 25.
        \item The model correctly identifies also the congruence classes modulo 27, for in this case $p=3$, $a=3$, $j=2$ and condition \eqref{identif_cond} is satisfied.
    \end{itemize}
\end{itemize}

In all of those situations in which accuracy 1.00 is not attained, the validity of the error pattern indicated in Experimental Observation \ref{EO_prime_power} can be checked by looking at the confusion matrices and noting that they share a common form. As in the examples provided in Subsection \ref{PatternA}, this common form consists of three diagonal blocks, differently colored in the examples below: the leftmost block, colored in blue, corresponds to class $[0]_m$ being correctly identified, while the rightmost block, colored in yellow, corresponds to the last $B$ classes being correctly identified. The substantial difference with respect to the prime moduli case lies in the central block, which now presents all of the non-zero entries on a peculiar diagonal pattern where the diagonal lines are $p^j$ places apart ($p$ and $j$ having here the same meaning as in the statement of Experimental Observation \ref{EO_prime_power}). 

Consider, for instance, Confusion Matrix \ref{B8m25k7}. In this case $m=25$, $B=8$ and $k=7$, therefore $p=5$, $a=2$, $j=1$ and in the central block we find all of the non-zero entries along diagonal lines (colored in light red and green) that are $p^j=5$ places apart from each other. More analogous examples are provided by Confusion Matrices  \ref{B16m25k7}, \ref{B8m27k7} and \ref{B16m27k7}.

\begin{figure*}[!t]
\centering
\includegraphics[width=8cm]{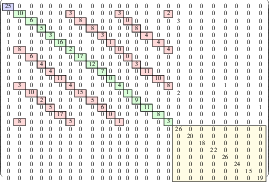}

\noindent\begin{minipage}{\textwidth}
\begin{matrice}\label{B8m25k7}
\centering
$m=25$, $B=8$, $k=7$.
\end{matrice}
\end{minipage}
\end{figure*}

\begin{figure*}[!t]
\centering
\includegraphics[width=8.7cm]{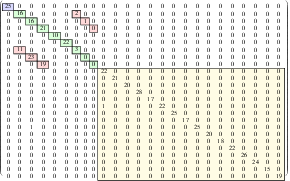}

\noindent\begin{minipage}{\textwidth}
\begin{matrice}\label{B16m25k7}
\centering
$m=25$, $B=16$, $k=15$.
\end{matrice}
\end{minipage}
\end{figure*}

% \begin{figure*}[!t]
% \centering
% \begin{minipage}{0.49\textwidth}
% \centering
% \includegraphics[width=8cm]{m25-B8-k7.eps}

% \noindent\begin{minipage}{7cm}
% \begin{matrice}\label{B8m25k7}
% \centering
% $m=25$, $B=8$, $k=7$.
% \end{matrice}
% \end{minipage}
% \end{minipage}
% \begin{minipage}{0.49\textwidth}
% \centering
% \includegraphics[width=8.7cm]{m25-B16-k15.eps}

% \noindent\begin{minipage}{8cm}
% \begin{matrice}\label{B16m25k7}
% \centering
% $m=25$, $B=16$, $k=15$.
% \end{matrice}
% \end{minipage}
% \end{minipage}
% \end{figure*}

% \begin{figure*}[!t]
% \centering
% \begin{minipage}{0.49\textwidth}
% \centering
% \includegraphics[width=8.4cm]{m27-B8-k7.eps}

% \noindent\begin{minipage}{7.5cm}
% \begin{matrice}\label{B8m27k7}
% \centering
% $m=27$, $B=8$, $k=7$.
% \end{matrice}
% \end{minipage}
% \end{minipage}
% \begin{minipage}{0.49\textwidth}
% \centering
% \includegraphics[width=8.7cm]{m27-B16-k15.eps}

% \noindent\begin{minipage}{7.5cm}
% \begin{matrice}\label{B16m27k7}
% \centering
% $m=27$, $B=16$, $k=15$.
% \end{matrice}
% \end{minipage}
% \end{minipage}
% \end{figure*}

\begin{figure*}[!t]
\centering
\includegraphics[width=8.4cm]{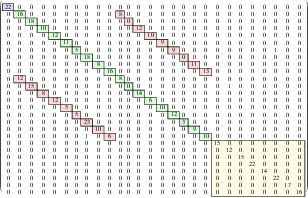}

\noindent\begin{minipage}{\textwidth}
\begin{matrice}\label{B8m27k7}
\centering
$m=27$, $B=8$, $k=7$.
\end{matrice}
\end{minipage}
\end{figure*}

\begin{figure*}[!t]
\centering
\includegraphics[width=8.7cm]{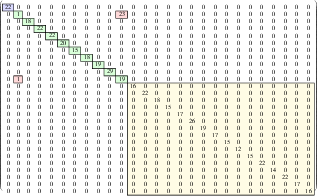}

\noindent\begin{minipage}{\textwidth}
\begin{matrice}\label{B16m27k7}
\centering
$m=27$, $B=16$, $k=15$.
\end{matrice}
\end{minipage}
\end{figure*}

\subsection{Splitting Moduli}\label{PatternC}
In this Subsection we present the final experiments supporting our general claims, and analyze the case of splitting moduli, that is, moduli $m$ whose prime factor decomposition involves at least two distinct primes. Table \ref{table_pattern_C} lists the accuracies attained by \PCM\ when $m\in\{6,10,12,14,15,18,20,21,22,24,26,28,30\}$, $B\in\{8,16,24\}$ and $k\in\{7,15\}$.

\begin{table}
\begin{center}
\caption{Accuracies attained by \PCM\ when $m\in\{2,3,\dots,30\}$ splits, $B\in\{8,16,24\}$ and $k\in\{7,15\}$.}
\label{table_pattern_C}
\begin{tabular}{c|c|cc|cc|} 
\hhline{~-----}
\multicolumn{1}{c|}{} & \multicolumn{1}{c|}{\cellcolor{gray!30}$\boldsymbol{B=8}$} &  \multicolumn{2}{c|}{\cellcolor{gray!30}$\boldsymbol{B=16}$} & \multicolumn{2}{c|}{\cellcolor{gray!30}$\boldsymbol{B=24}$}  \\
\hline
\multicolumn{1}{|c|}{\cellcolor{gray!30}$\boldsymbol{m}$} & $\cellcolor{gray!30}\boldsymbol{k=7}$ & \cellcolor{gray!30}$\boldsymbol{k=7}$ & \cellcolor{gray!30}$\boldsymbol{k=15}$ & \cellcolor{gray!30}$\boldsymbol{k=7}$ & \cellcolor{gray!30}$\boldsymbol{k=15}$\\
\hline
\multicolumn{1}{|c|}{\cellcolor{gray!30}\bf 6} & 1.00 & 1.00 & 1.00 & 1.00 & 1.00\\
\hline
\multicolumn{1}{|c|}{\cellcolor{gray!30}\bf 10} & 1.00 & 1.00 & 1.00 & 1.00 & 1.00\\
\hline
\multicolumn{1}{|c|}{\cellcolor{gray!30}\bf 12} & 1.00 & 1.00 & 1.00 & 1.00 & 1.00\\
\hline
\multicolumn{1}{|c|}{\cellcolor{gray!30}\bf 14} & 1.00 & 1.00 & 1.00 & 1.00 & 1.00\\
\hline
\multicolumn{1}{|c|}{\cellcolor{gray!30}\bf 15} & 1.00 & 1.00 & 1.00 & 1.00 & 1.00\\
\hline
\multicolumn{1}{|c|}{\cellcolor{gray!30}\bf 18} & 1.00 & 1.00 & 1.00 & 1.00 & 1.00\\
\hline
\multicolumn{1}{|c|}{\cellcolor{gray!30}\bf 20} & 1.00 & 1.00 & 1.00 & 1.00 & 1.00\\
\hline
\multicolumn{1}{|c|}{\cellcolor{gray!30}\bf 21} & 1.00 & 1.00 & 1.00 & 1.00 & 1.00\\
\hline
\multicolumn{1}{|c|}{\cellcolor{gray!30}\bf 22} & 0.91 & 0.99 & 0.99 & 1.00 & 1.00\\
\hline
\multicolumn{1}{|c|}{\cellcolor{gray!30}\bf 24} & 1.00 & 1.00 & 1.00 & 1.00 & 1.00\\
\hline
\multicolumn{1}{|c|}{\cellcolor{gray!30}\bf 26} & 0.76 & 1.00 & 1.00 & 0.99 & 1.00\\
\hline
\multicolumn{1}{|c|}{\cellcolor{gray!30}\bf 28} & 1.00 & 1.00 & 1.00 & 1.00 & 1.00\\
\hline
\multicolumn{1}{|c|}{\cellcolor{gray!30}\bf 30} & 1.00 & 1.00 & 1.00 & 1.00 & 1.00\\
\hline
\end{tabular}
\end{center}
\end{table}

It can be immediately seen from the results that the model can identify the congruence classes modulo some integer $m=p_1^{\ell_1}\cdots p_\alpha^{\ell_\alpha}$ if it can identify the congruence classes modulo each $p_i^{\ell_i}$. This is the case, for example, of modulo $m=30=2\cdot 3\cdot 5$: the model can positively deal with this modulo even using a relatively short sequence of length $B=8$ since this value of $B$ allows the identification of the congruence classes modulo each of the primes 2, 3 and 5. 

On the other hand, if the model cannot correctly identify  the congruence classes modulo some of the component prime powers, then it cannot identify either the congruence classes modulo the whole number. The confusion matrices, in this situation, mirror the error pattern of the unknown component as many times as indicated by the known component. This is the case of $m=22=2\cdot 11$ and $m=26=2\cdot 13$ when a sequence of length $B=8$ is employed. Indeed we see that Confusion Matrix \ref{B8m22k7} presents the same three-block structure of Confusion Matrix \ref{B8m11k7} replicated twice, while Confusion Matrix \ref{B8m26k7} presents the same three-block structure of Confusion Matrix \ref{B8m13k7} also replicated twice.

\begin{figure*}[!t]
\centering
\includegraphics[width=7.7cm]{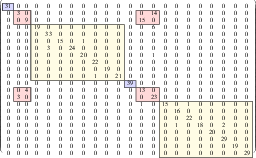}

\noindent\begin{minipage}{\textwidth}
\begin{matrice}\label{B8m22k7}
\centering
$m=22$, $B=8$, $k=7$.
\end{matrice}
\end{minipage}
\end{figure*}

\begin{figure*}[!t]
\centering
\includegraphics[width=8cm]{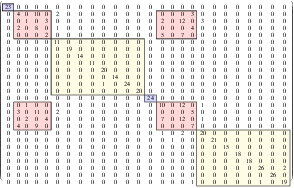}

\noindent\begin{minipage}{\textwidth}
\begin{matrice}\label{B8m26k7}
\centering
$m=26$, $B=8$, $k=7$.
\end{matrice}
\end{minipage}
\end{figure*}

\section{Towards Theoretical Explanation}\label{Theoretical_consequences}
The outcomes of the experiments presented in Section \ref{Experimental_Results} are surprising and very interesting. Not only they reveal precise relationships between the modulo $m$ whose equivalence classes we aim to classify and the locality hyperparameter $B$, relationships that permit to predict whether \PCM\ will fulfill the task, but also, when this is not the case, they show clear error patterns that highlight the features that the model cannot capture. 

Furthermore, the Experimental Observations drawn from the experiments' results can be elaborated mathematically to deduce new verifiable hypotheses on the behavior of \PCM. This is exemplified by the following elementary result that clarifies the content of Experimental Observation  \ref{EO_prime_power} explaining, for a given prime number $p$, which are the moduli of the form $m=p^n$ that \PCM\ can treat correctly. 

\begin{theorem}
Assume the validity of Experimental Observation  \ref{EO_prime_power}. Let $B\geq 2$ and let $p$ be a prime. Let $i_0\geq 0$ be the unique integer such that $p^{i_0}< B+2\leq p^{i_0+1}$. Then:
\begin{enumerate}
	\item \PCM\ solves the task of identifying the congruence classes modulo $p^i$ if $i\leq i_0$.
		\item \PCM\ solves the task of identifying the congruence classes modulo $p^{i_0+1}$ if and only if
	\begin{equation}\label{disc_cond_p}
		p^{i_0+1}-p^{i_0}-1\leq B.
	\end{equation}
	\item \PCM\ does not solve the task of identifying the congruence classes modulo $p^i$ if $i\geq i_0+2$.
\end{enumerate}
\end{theorem}
\begin{proof}
1. If $i\leq i_0$, then \PCM\ solves the task of identifying the congruence classes modulo $p^i$  by Experimental Observation  \ref{EO_prime_power}(1).

2. The largest integer $j\in\{0,1,\dots, i_0\}$ such that \PCM\ solves the task of identifying the congruence classes modulo $p^j$ is $i_0$. If $p^{i_0+1} >B+2$ then by Experimental Observation  \ref{EO_prime_power}(2)  \PCM\ solves the task of identifying the congruence classes modulo $p^{i_0+1}$ if and only if $p^{i_0+1}-B-1\leq p^{i_0}$, that is if and only if \eqref{disc_cond_p} holds. If $p^{i_0+1}=B+2$ then by  Experimental Observation  \ref{EO_prime_power}(1)  \PCM\ solves the task of identifying the congruence classes modulo $p^{i_0+1}$ and \eqref{disc_cond_p} is verified in this case because
$$
p^{i_0+1}-p^{i_0}-1=B-(p^{i_0}-1)\leq B.
$$

3. Let $i\geq i_0+2$, and let $j\in\{0,1,\dots,i-1\}$ be the largest integer such that the model can correctly identify the congruence classes modulo $p^j$. By Experimental Observation  \ref{EO_prime_power}(2), it suffices to show that the inequality $p^i-B-1\leq p^j$ cannot hold. If, by contradiction, we had $p^i-B-1\leq p^j$, then
\begin{equation*}\begin{split}
p^i&\leq p^j+B+1\\
  &\leq  p^{i-1}+p^{i_0+1}-1\\
  &\leq 2\cdot p^{i-1}-1\\
  &<2\cdot p^{i-1}
\end{split}\end{equation*}
because $j\leq i-1$, $B+1\leq p^{i_0+1}-1$ and $i_0+1\leq i-1$. Hence
$$
p^{i-1}(p-2)< 0,
$$
and this is impossible since $p\geq 2$.
\end{proof}

\begin{corollary}
If $p=2$, then \PCM\ solves the task of identifying the congruence classes modulo $2^i$ if and only if $i\leq i_0+1$.
\end{corollary}
\begin{proof}
We need only observe that condition \eqref{disc_cond_p} is always verified if $p=2$. Indeed, $B+2>2^{i_0}$, therefore $B+1\geq 2^{i_0}$ and
$$
2^{i_0+1}-2^{i_0}-1 = 2^{i_0}-1\leq B,
$$
as claimed.
\end{proof}

\section{Conclusions and Future Work}\label{Conclusion}
The main goal of this paper is to exploit the Scientific Method to explain the behaviors of Neural Networks by providing a mathematical model based on empirical evidence that highlights the relationships between their inputs and outputs in a rigorous algebraic way.

The empirical observations allow us to draw the theoretical consequences for describing and explaining the behaviors of \PCM. Thus, a mathematical formulation was derived that models patterns of interest to explain when and why \PCM\ succeeds in performing the task and, if not, what error pattern it follows.

This theoretical approach was successful for \PCM, a neural model built within a controlled environment where the features of the data are arithmetical and mathematically related to each other. Indeed, the most striking conclusion that we draw from the analysis carried out in the previous sections is the possibility of formulating rules that describe in mathematical terms the behaviors of \PCM, making the outcomes of the experiments fully explainable and predictable. Our approach therefore suggests that \textquotedblleft explainability\textquotedblright\ should be interpreted not only as an a posteriori analysis aimed at understanding which features of the dataset are responsible for a particular outcome of the network, but more importantly as the development of a mathematical model that relates all the features of the data to the possible outcomes, in order to allow a prediction of the choices that the network will make and an {\em a priori} explanation of its behaviors.

The proposed methodology and its experimental validation have proven promising for defining the first theoretical approach to explain convolutional networks. It would be interesting to extend the study using other datasets and other architectures, as we believe that this will shed new light on the inner mechanisms governing Neural Networks, constituting the basis for a new Theory of Explainability.

\section*{Acknowledgments}
The authors heartily thank Massimo Bertini of Mathema srl for his support during the experiments and his precious help in reviewing and testing the code.

% \begin{figure*}[!t]
% \centering
% \begin{minipage}{0.49\textwidth}
% \centering
% \includegraphics[width=7.7cm]{m22-B8-k7.eps}

% \noindent\begin{minipage}{7cm}
% \begin{matrice}\label{B8m22k7}
% \centering
% $m=22$, $B=8$, $k=7$.
% \end{matrice}
% \end{minipage}
% \end{minipage}
% \begin{minipage}{0.49\textwidth}
% \centering
% \includegraphics[width=8cm]{m26-B8-k7.eps}

% \noindent\begin{minipage}{8cm}
% \begin{matrice}\label{B8m26k7}
% \centering
% $m=26$, $B=8$, $k=7$.
% \end{matrice}
% \end{minipage}
% \end{minipage}
% \end{figure*}


\appendix
\section{Hyperparameters Tuning}
This Appendix is devoted to tuning the various hyperparameters of \PCM\ (cf. Subsection \ref{Hyperparameters}). All our experiments were conducted using a dedicated machine with hardware specifications Intel Core i9 processor ($20\times 3.7$ GHz) and 16 GB RAM.

\subsection{Convolutional Hyperparameters}\label{App_Conv_Hyp}
In most applications several convolutional layers are employed, each with a rather large number of channels (typically $2^6$-$2^8$) that are supposed to capture different features of the dataset (cf. \citep{ajit2020review}, \citep{almhaithawi2023construction}, \citep{chen2021review}, \citep{li2021survey}, \citep{peng2020semicdnet}, \citep{soffer2019convolutional}). This philosophy suggests to first choose $\lambda=2$ and compare the model's performances with values of $C_1$ and $C_2$ that range from 1 to 128.

To this end we employ seven different versions of \PCM\ to classify the congruence classes modulo $m\in\{2,3,\dots,10\}$. These versions, realized with $\lambda=2$ as indicated in Subsection \ref{Hyperparameters}, are trained for $t=10$ epochs using $r=400$ batches of size $s=256$ each and are characterized by different values of the hyperparameters $C_1$ and $C_2$. The accuracies achieved during the validation process are listed in Table~\ref{optimizing_channels}.

\begin{table*}
\begin{center}
\caption{Accuracies attained by different versions of \PCM\ in identifying the congruence classes modulo $m\in\{2,3,\dots,10\}$ with $B=8$; each version of the model uses two convolutional layers and is characterized by the indicated values of $C_1$ and $C_2$.}
\label{optimizing_channels}
\begin{tabular}{|c|c|ccccccccc|} 
\hhline{~~---------}
\multicolumn{2}{c|}{} & \multicolumn{9}{c|}{\cellcolor{gray!30}$\boldsymbol{m}$}\\
\hline
\cellcolor{gray!30} $\boldsymbol{C_1}$ & \cellcolor{gray!30}$\boldsymbol{C_2}$ & \cellcolor{gray!30} \bf 2 & \cellcolor{gray!30} \bf 3 & \cellcolor{gray!30} \bf 4 & \cellcolor{gray!30} \bf 5 & \cellcolor{gray!30} \bf 6 &\cellcolor{gray!30} \bf 7 & \cellcolor{gray!30} \bf 8 & \cellcolor{gray!30} \bf 9 & \cellcolor{gray!30} \bf 10 \\  
\hline
64 & 128 & 1.00 & 1.00 & 1.00 &  1.00 & 1.00 & 1.00 & 1.00 & 1.00 & 1.00 \\ 
\hline
32 & 64  & 1.00 & 1.00 & 1.00 & 1.00 & 1.00 & 1.00 & 1.00 & 1.00 & 1.00 \\ 
\hline
16 & 32  & 1.00 & 1.00 & 1.00 & 1.00 & 1.00 & 1.00 & 1.00 & 1.00 & 1.00 \\ 
\hline
8 & 16  & 1.00 & 1.00 & 1.00 & 1.00 & 1.00 & 1.00 & 1.00 & 1.00 & 1.00 \\ 
\hline
4 & 8  & 1.00 & 1.00 & 1.00 & 1.00 & 1.00 & 1.00 & 1.00 & 1.00 & 1.00 \\
\hline
2 & 4  & 1.00 & 1.00 & 1.00 & 1.00 & 1.00 & 1.00 & 1.00 & 1.00 & 1.00 \\ 
\hline
1 & 2  & 1.00 & 1.00 & 1.00 & 0.90 & 1.00 & 1.00 & 1.00 & 1.00 & 0.90 \\ 
\hline
\end{tabular}
\end{center}
\end{table*}

Next, we analyze the necessity of using two convolutional layers. The strategy that we adopt is similar to the previous one: we construct two versions of \PCM\ as outlined in Subsection \ref{Hyperparameters} with $\lambda=1$. The number $C_1=C$ of channels within this layer is set to $C=4$ for the first version and to $C=2$ for the second. The accuracies obtained after validating are listed in Table~\ref{optimizing_conv_layers}.

\begin{table*}
\begin{center}
\caption{Accuracies attained by different versions of \PCM\ in identifying the congruence classes modulo $m\in\{2,3,\dots,10\}$ with $B=8$; each version of the model uses one convolutional layer and is characterized by the indicated values of the hyperparameter $C_1=C$.}
\label{optimizing_conv_layers}
\begin{tabular}{|c|ccccccccc|} 
\hhline{~---------}
\multicolumn{1}{c|}{} & \multicolumn{9}{c|}{\cellcolor{gray!30}$\boldsymbol{m}$}\\
\hline
\cellcolor{gray!30} $\boldsymbol{C}$ & \cellcolor{gray!30} \bf 2 & \cellcolor{gray!30} \bf 3 & \cellcolor{gray!30} \bf 4 & \cellcolor{gray!30} \bf 5 & \cellcolor{gray!30} \bf 6 &\cellcolor{gray!30} \bf 7 & \cellcolor{gray!30} \bf 8 & \cellcolor{gray!30} \bf 9 & \cellcolor{gray!30} \bf 10 \\  
\hline
4  & 1.00 & 1.00 & 1.00 & 1.00 & 1.00 & 1.00 & 1.00 & 1.00 & 1.00 \\ 
\hline
2  & 1.00 & 1.00 & 1.00 & 1.00 & 1.00 & 0.99 & 1.00 & 1.00 & 0.99 \\ 
\hline
\end{tabular}
\end{center}
\end{table*}

The conclusion that we draw from these experiments is that the number of channels and convolutional layers typically employed in the literature is excessive in our case. The reason that can be adduced to justify this phenomenon lies in the intrinsic nature of our dataset. Namely, in many standard applications of convolutional models, the data possess several features, most of which vary independently and may intertwine in unexpected ways. In order to extract these features and elaborate the input, the network needs to process the data through multiple parallel channels: different channels capture different features, and consequently the more channels are employed in the architecture, the more features can be captured. When working with natural numbers, however, the situation changes substantially: the features are no longer subjective, and the relationships between them cannot behave too wildly. In fact, everything is subject to strict arithmetic rules, and the network's understanding of these rules is all that matters to accomplish the proposed task.
Moreover, the experiments presented in Section \ref{Experimental_Results} show that the network's comprehension of the arithmetic structure of natural numbers is more likely due to the peculiar convolutional characteristics of the architecture (i.e. the fact that numbers are processed in sequences acted on by a kernel) rather than to the presence of a large number of channels operating in parallel.

In light of the previous discussion and experiments, we choose in our model
$$
\boxed{\lambda=1} \qquad\text{and}\qquad \boxed{C=4}
$$

\subsection{Training Hyperparameters}\label{App_Training_Hyp}
In order to optimize the number of epochs for the training process, we apply the model constructed so far to classify the congruence classes modulo an integer $m\in\{11,13,17,19\}$ using a sequence of length $B=8$ and a set of $r=400$ batches of size $s=256$ each. We then measure the time it takes the model to be trained and determine the accuracy achieved by validating after $t=10$ and $t=50$ epochs. The results are summarized in Table \ref{B8optimizing_epochs}. 

Observe that the employed values of $m$ are all prime numbers. As can be seen from the experiments presented in Section \ref{Experimental_Results}, prime moduli are the most challenging for the model, as in these cases there exists a bunch of classes that cannot be distinguished from each other. Therefore, we expect these moduli to be the most demanding in terms of number of epochs for the training.

\begin{table}
\begin{center}
\caption{Accuracies attained by \PCM\ in identifying the congruence classes modulo $m\in\{11,13,17,19\}$ with $B=8$ after $t=10$ and $t=50$ epochs.}
\label{B8optimizing_epochs}
\begin{tabular}{|c|c|cccc|} 
\hhline{~~----}
\multicolumn{2}{c|}{} & \multicolumn{4}{c|}{\cellcolor{gray!30}$\boldsymbol{m}$}\\
\hline
\cellcolor{gray!30} $\boldsymbol{t}$ & \cellcolor{gray!30} \bf Training time (min) & \cellcolor{gray!30}\bf 11 & \cellcolor{gray!30}\bf 13 & \cellcolor{gray!30}\bf 17 & \cellcolor{gray!30} \bf 19 \\  
\hline
 10 &  25 & 0.92 &  0.76 &  0.59 &  0.54 \\ 
\hline
50 & 128 & 0.92 & 0.77 &  0.61 &  0.56 \\ 
\hline
\end{tabular}
\end{center}
\end{table}

From Table \ref{B8optimizing_epochs} we can infer that the accuracy of the results does not improve significantly when passing from epoch 10 to epoch 50 (this observation could also be made more precise by comparing the confusion matrices, which indeed show the same pattern described in Subsection \ref{PatternA} in all cases, both at epoch 10 and at epoch 50, without relevant differences); however, the building time increases linearly with the number of epochs, ranging from a minimum of 25 minutes for a 10-epoch training to a maximum of more than two hours for a 50-epoch training.
In light of these observations, we choose to train our model for a total of
$$
\boxed{t=10}
$$
epochs.

The choice of the optimal values for $r$ (the number of batches) and $s$ (the size of each batch) proceeds analogously: we train the model constructed so far to classify the congruence classes modulo $m\in\{13,19\}$ using a sequence of length $B=8$ and various values for $r$ and $s$; the accuracy achieved by the model, as well as the training time, the total size of the training set and its percentage over the whole dataset, are listed in Table \ref{B8optimizing_batches}.

A quick look at Table \ref{B8optimizing_batches} shows that a training set size of 6~400 elements is too small to achieve acceptable results in terms of accuracy, while a training set size of 204~800 elements is excessive in terms of time, since smaller training sets achieve the same accuracy faster. A training set of 25~600 elements appears optimal for low values of $m$ since it achieves good accuracies and, more substantially, since the corresponding confusion matrices already reveal the pattern of Subsection \ref{PatternA}; however, for large values of $m$, this dataset loses significance because it does not contain enough representatives for each congruence class. Among the intermediate choices with a dataset size of 102~400 elements, the various values of $r$ and $s$ do not modify the outputs relevantly; we therefore choose for our \PCM
$$
\boxed{r=400}\qquad\text{and} \qquad \boxed{s=256}
$$

\subsection{Locality Hyperparameters}\label{App_Loc_Hyp}

In this Subsection we select the values through which we make the locality hyperparameters range in the experiments of Section \ref{Experimental_Results}.

We have observed several times that the hyperparameter $B$ controls the network's learning of the additive structure of $\mathbb{N}$. In particular, this means that its values must be assigned in relation to the task we are addressing: for example, it is expected (and the experiments in Section \ref{Experimental_Results} have shown) that the network does not necessarily need to identify all congruence classes modulo $m$ when the value of $B$ is small compared to $m$. Furthermore, since we aim to analyze and explain all the various behaviors that the model can exhibit, the values of $B$ and $m$ must be carefully chosen to exemplify them all exhaustively.

The only constraint we impose on ourselves in choosing the highest values of $B$ and $m$ concerns the time required to train the model (in fact, we do not expect the value of $m$ to affect the training time). To accomplish this goal, we set $m=20$, $m=30$ and look for the highest value of $B$ that allows a training time not exceeding one hour. The results that we obtain are listed in Table~\ref{choiceB}, where we also vary the size $k$ of the kernel.

\begin{table*}
	\begin{center}
		\caption{Accuracies attained by \PCM\ in identifying the congruence classes\\ modulo $m\in\{13,19\}$ with $B=8$ using $r$ batches of size $s$.}
		\label{B8optimizing_batches}
		\begin{tabular}{|cc|c|c|c|cc|} 
			\hhline{~~~~~--}
			\multicolumn{5}{c|}{} & \multicolumn{2}{c|}{\cellcolor{gray!30}$\boldsymbol{m}$} \\
			\hline 
			\cellcolor{gray!30} $\boldsymbol{r}$ & \cellcolor{gray!30} $\boldsymbol{s}$  & \cellcolor{gray!30} \bf Training set size & \cellcolor{gray!30} \bf Percentage & \cellcolor{gray!30} \bf Training time (min) & \cellcolor{gray!30} \bf 13 & \cellcolor{gray!30} \bf 19 \\  
			\hline
			100 & 64  & 6\,400 & 0.82\% & 2 & 0.63 &  0.36 \\ 
			\hline
			200 & 128 & 25\,600 & 3.3\% & 6 & 0.73 & 0.51 \\ 
			\hline\hline
			200 & 512 & 102\,400 & 13\% & 26 & 0.77 & 0.54 \\
			\hline
			\bf 400 & \bf 256 & \bf 102\,400 & \bf 13\% & \bf 25 & \bf 0.76 & \bf 0.54 \\ 
			\hline
			800 & 128 & 102\,400 & 13\% & 25 & 0.76 & 0.53 \\ 
			\hline\hline
			400 & 512 & 204\,800 & 26\% & 50 &  0.79 & 0.46 \\  
			\hline
		\end{tabular}
	\end{center}
\end{table*}

\begin{table}
\begin{center}
\caption{Comparison of the 10-epoch training time of \PCM\\for $m\in\{20,30\}$ and different values of $B$ and $k$.}
\label{choiceB}
\begin{tabular}{|cc|cc|} 
\hhline{~~--}
\multicolumn{2}{c|}{} & \multicolumn{2}{c|}{\cellcolor{gray!30} \bf Training time (min) }\\ 
\hline
\cellcolor{gray!30} $\boldsymbol{B}$ & \cellcolor{gray!30} $\boldsymbol{k}$ & \multicolumn{1}{c}{\cellcolor{gray!30} $\boldsymbol{m=20}$}  &
\multicolumn{1}{c|}{\cellcolor{gray!30} $\boldsymbol{m=30}$} \\  
\hline
8 &  7 & 25 & 25 \\ 
\hline
16 & 7 & 43 & 43\\ 
\hline
16 &  15 & 50 & 49 \\ 
\hline
24 & 7 & 62 & 62 \\ 
\hline
24 &  15 & 69 & 70 \\ 
\hline
\end{tabular}
\end{center}
\end{table}

\begin{table*}
\begin{center}
\caption{Accuracies attained by \PCM\ in identifying the congruence classes modulo $m=7$, $17$, $23$ with $B=8$, $16$, $24$ respectively, using kernels of various sizes. }
\label{optimizing_kernel}
\begin{tabular}{cc|ccccccc|} 
\hhline{~~-------}
\multicolumn{2}{c|}{} & \multicolumn{7}{c|}{\cellcolor{gray!30}$\boldsymbol{k}$}  \\
\hline
\multicolumn{1}{|c}{\cellcolor{gray!30}$\boldsymbol{m}$} & $\cellcolor{gray!30}\boldsymbol{B}$ & \cellcolor{gray!30} \bf 3 & \cellcolor{gray!30} \bf 5 & \cellcolor{gray!30} \bf 7 & \cellcolor{gray!30} \bf 9 & \cellcolor{gray!30} \bf 11 & \cellcolor{gray!30} \bf 13 & \cellcolor{gray!30} \bf 15 \\
\hline
\multicolumn{1}{|c}{7} & 8 & 0.99 & 1.00 & 1.00 & $-$ & $-$ & $-$ & $-$ \\
\hline
\multicolumn{1}{|c}{17} & 16 & 0.07 & 0.85 & 1.00 &  1.00 & 1.00 & 1.00 & 1.00 \\
\hline
\multicolumn{1}{|c}{23} & 24 & 0.42 & 0.87 & 0.92 & 0.97 & 0.96 & $\sim$ 1.00 & 1.00\\
\hline
\end{tabular}
\end{center}
\end{table*}

\begin{figure*}
    \centering
    \includegraphics[width=0.5\textwidth]{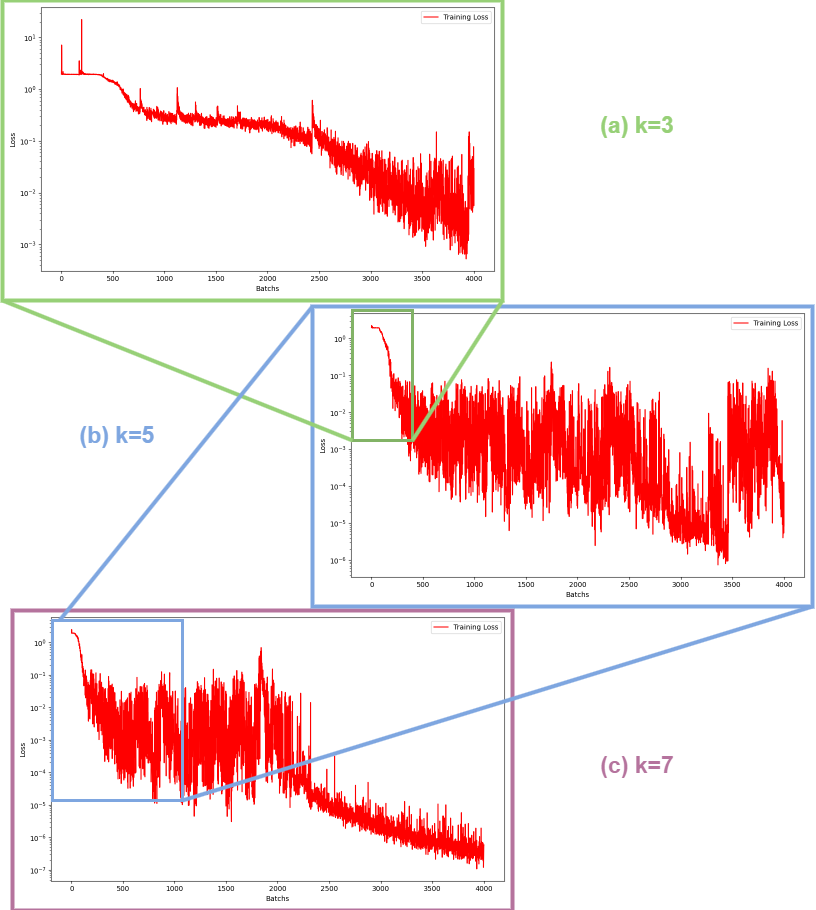}
    \caption{Graphs of \PCM's training losses when $m=7$, $B=8$ and {\rm(a)} $k=3$, {\rm(b)} $k=5$, {\rm(c)} $k=7$.}
    \label{loss_figures}
\end{figure*}

The results show that the training time does not depend on the value of $m$ and that it firstly exceeds one hour when $B=24$; thus, this value will be our upper bound for the hyperparameter $B$. As for the task, we take $m=30$ as the maximum modulo.

In order to perform our experiments, we will assign to $m$ all possible integer values from $2$ to $30$; however, we deem it redundant to apply a similar systematicity to the hyperparameter $B$ making it take all possible values from $2$ to $24$. Indeed, the various behaviors displayed by the model are exhaustively exemplified and explained by conducting our analysis with just three values for $B$, namely $B=8$, $16$ and $24$ (cf. Section \ref{Experimental_Results}).

Finally, we choose the values for the kernel size $k$. To this end, for each of the three chosen values of $B$, we test the model on the largest prime modulo $m$ whose classes can be correctly identified (i.e., $m=7$ when $B=8$, $m=17$ when $B=16$, and $m=23$ when $B=24$): we look for the kernel sizes that allow the model to achieve the perfect accuracy of 1.00. The results of this analysis are presented in Table \ref{optimizing_kernel}.

When $B=8$, we chose $k=7$, even though this is not the smallest size that achieves accuracy 1.00, since in this situation the training loss approaches zero faster than in the other two cases (cf. Figure \ref{loss_figures}).

When $B=16$ (resp. $B=24$), the smallest kernel that achieves accuracy 1.00 has size $k=7$ (resp. $k=15$), and therefore this will be our choice. However, for completeness, we shall also perform the experiments with $k=15$ (resp. $k=7$) to compare the various outputs of the model.

\bibliographystyle{elsarticle-num} 
\bibliography{References}

\begin{thebibliography}{10}
\expandafter\ifx\csname url\endcsname\relax
  \def\url#1{\texttt{#1}}\fi
\expandafter\ifx\csname urlprefix\endcsname\relax\def\urlprefix{URL }\fi
\expandafter\ifx\csname href\endcsname\relax
  \def\href#1#2{#2} \def\path#1{#1}\fi

\bibitem{dong2015image}
C.~Dong, C.~C. Loy, K.~He, X.~Tang, Image super-resolution using deep
  convolutional networks, IEEE transactions on pattern analysis and machine
  intelligence 38~(2) (2015) 295--307.

\bibitem{li2022comprehensive}
Z.~Li, B.~Xia, J.~Zhang, C.~Wang, B.~Li, {A comprehensive survey on
  data-efficient GANs in image generation}, arXiv preprint arXiv:2204.08329
  (2022).

\bibitem{ning2020multi}
X.~Ning, F.~Nan, S.~Xu, L.~Yu, L.~Zhang, {Multi-view frontal face image
  generation: a survey}, Concurrency and Computation: Practice and Experience
  (2020) e6147.

\bibitem{Shen2020InterFaceGANIT}
Y.~Shen, C.~Yang, X.~Tang, B.~Zhou, {InterFaceGAN: Interpreting the
  Disentangled Face Representation Learned by GANs}, IEEE Transactions on
  Pattern Analysis and Machine Intelligence 44 (2020) 2004--2018.

\bibitem{de2021survey}
G.~H. de~Rosa, J.~P. Papa, {A survey on text generation using generative
  adversarial networks}, Pattern Recognition 119 (2021) 108098.

\bibitem{briot2020deep}
J.-P. Briot, F.~Pachet, {Deep learning for music generation: challenges and
  directions}, Neural Computing and Applications 32~(4) (2020) 981--993.

\bibitem{baltruvsaitis2018multimodal}
T.~Baltru{\v{s}}aitis, C.~Ahuja, L.-P. Morency, Multimodal machine learning: A
  survey and taxonomy, IEEE transactions on pattern analysis and machine
  intelligence 41~(2) (2018) 423--443.

\bibitem{chen2017deeplab}
L.-C. Chen, G.~Papandreou, I.~Kokkinos, K.~Murphy, A.~L. Yuille, Deeplab:
  Semantic image segmentation with deep convolutional nets, atrous convolution,
  and fully connected crfs, IEEE transactions on pattern analysis and machine
  intelligence 40~(4) (2017) 834--848.

\bibitem{minaee2021image}
S.~Minaee, Y.~Boykov, F.~Porikli, A.~Plaza, N.~Kehtarnavaz, D.~Terzopoulos,
  Image segmentation using deep learning: A survey, IEEE transactions on
  pattern analysis and machine intelligence 44~(7) (2021) 3523--3542.

\bibitem{bakator2018deep}
M.~Bakator, D.~Radosav, {Deep learning and medical diagnosis: A review of
  literature}, Multimodal Technologies and Interaction 2~(3) (2018) 47.

\bibitem{soffer2019convolutional}
S.~Soffer, A.~Ben-Cohen, O.~Shimon, M.~M. Amitai, H.~Greenspan, E.~Klang,
  {Convolutional neural networks for radiologic images: a radiologist’s
  guide}, Radiology 290~(3) (2019) 590--606.

\bibitem{mumali2022artificial}
F.~Mumali, {Artificial neural network-based decision support systems in
  manufacturing processes: A systematic literature review}, Computers \&
  Industrial Engineering (2022) 107964.

\bibitem{chen2022neural}
J.~Chen, Z.~Chen, C.~Yao, H.~Qiao, {Neural manifold modulated continual
  reinforcement learning for musculoskeletal robots}, IEEE Transactions on
  Cognitive and Developmental Systems (2022).

\bibitem{chen2022recurrent}
J.~Chen, X.~Huang, X.~Wang, H.~Qiao, {Recurrent Neural Network based Partially
  Observed Feedback Control of Musculoskeletal Robots}, in: 2022 International
  Conference on Advanced Robotics and Mechatronics (ICARM), IEEE, 2022, pp.
  12--18.

\bibitem{anitescu2019artificial}
C.~Anitescu, E.~Atroshchenko, N.~Alajlan, T.~Rabczuk, {Artificial neural
  network methods for the solution of second order boundary value problems},
  Computers, Materials and Continua 59~(1) (2019) 345--359.

\bibitem{Lample2020Deep}
G.~Lample, F.~Charton, \href{https://openreview.net/forum?id=S1eZYeHFDS}{{Deep
  Learning For Symbolic Mathematics}}, in: International Conference on Learning
  Representations, 2020.
\newline\urlprefix\url{https://openreview.net/forum?id=S1eZYeHFDS}

\bibitem{davies2021advancing}
A.~Davies, P.~Veli{\v{c}}kovi{\'c}, L.~Buesing, S.~Blackwell, D.~Zheng,
  N.~Toma{\v{s}}ev, R.~Tanburn, P.~Battaglia, C.~Blundell, A.~Juh{\'a}sz,
  et~al., {Advancing mathematics by guiding human intuition with AI}, Nature
  600~(7887) (2021) 70--74.

\bibitem{aggarwal2021generative}
A.~Aggarwal, M.~Mittal, G.~Battineni, {Generative adversarial network: An
  overview of theory and applications}, International Journal of Information
  Management Data Insights 1~(1) (2021) 100004.

\bibitem{asperti2021survey}
A.~Asperti, D.~Evangelista, E.~Loli~Piccolomini, {A survey on variational
  autoencoders from a green AI perspective}, SN Computer Science 2~(4) (2021)
  301.

\bibitem{benfenati2023singular1}
A.~Benfenati, A.~Marta, {A singular Riemannian geometry approach to Deep Neural
  Networks I. Theoretical foundations}, Neural Networks 158 (2023) 331--343.

\bibitem{benfenati2023singular2}
A.~Benfenati, A.~Marta, {A singular Riemannian geometry approach to deep neural
  networks II. Reconstruction of 1-D equivalence classes}, Neural Networks 158
  (2023) 344--358.

\bibitem{chen2018metrics}
N.~Chen, A.~Klushyn, R.~Kurle, X.~Jiang, J.~Bayer, P.~Smagt, {Metrics for deep
  generative models}, in: International Conference on Artificial Intelligence
  and Statistics, PMLR, 2018, pp. 1540--1550.

\bibitem{hauser2017principles}
M.~Hauser, A.~Ray, {Principles of Riemannian geometry in neural networks},
  Advances in neural information processing systems 30 (2017).

\bibitem{fetty2020latent}
L.~Fetty, M.~Bylund, P.~Kuess, G.~Heilemann, T.~Nyholm, D.~Georg,
  T.~L{\"o}fstedt, {Latent space manipulation for high-resolution medical image
  synthesis via the StyleGAN}, Zeitschrift f{\"u}r Medizinische Physik 30~(4)
  (2020) 305--314.

\bibitem{giardina2022naive}
A.~Giardina, S.~S. Paria, A.~Kaustubh, {A naive method to discover directions
  in the StyleGAN2 latent space}, arXiv preprint arXiv:2203.10373 (2022).

\bibitem{arvanitidis2017latent}
G.~Arvanitidis, L.~K. Hansen, S.~Hauberg, {Latent space oddity: on the
  curvature of deep generative models}, arXiv preprint arXiv:1710.11379 (2017).

\bibitem{abiodun2018state}
O.~I. Abiodun, A.~Jantan, A.~E. Omolara, K.~V. Dada, N.~A. Mohamed, H.~Arshad,
  {State-of-the-art in artificial neural network applications: A survey},
  Heliyon 4~(11) (2018).

\bibitem{o2015introduction}
K.~O'Shea, R.~Nash, {An introduction to convolutional neural networks}, arXiv
  preprint arXiv:1511.08458 (2015).

\bibitem{li2021survey}
Z.~Li, F.~Liu, W.~Yang, S.~Peng, J.~Zhou, {A survey of convolutional neural
  networks: analysis, applications, and prospects}, IEEE transactions on neural
  networks and learning systems (2021).

\bibitem{kiranyaz20211d}
S.~Kiranyaz, O.~Avci, O.~Abdeljaber, T.~Ince, M.~Gabbouj, D.~J. Inman, {1D
  convolutional neural networks and applications: A survey}, Mechanical systems
  and signal processing 151 (2021) 107398.

\bibitem{goebel2018explainable}
R.~Goebel, A.~Chander, K.~Holzinger, F.~Lecue, Z.~Akata, S.~Stumpf,
  P.~Kieseberg, A.~Holzinger, {Explainable AI: the new 42?}, in: Machine
  Learning and Knowledge Extraction: Second IFIP TC 5, TC 8/WG 8.4, 8.9, TC
  12/WG 12.9 International Cross-Domain Conference, CD-MAKE 2018, Hamburg,
  Germany, August 27--30, 2018, Proceedings 2, Springer, 2018, pp. 295--303.

\bibitem{dovsilovic2018explainable}
F.~K. Do{\v{s}}ilovi{\'c}, M.~Br{\v{c}}i{\'c}, N.~Hlupi{\'c}, {Explainable
  artificial intelligence: A survey}, in: 2018 41st International convention on
  information and communication technology, electronics and microelectronics
  (MIPRO), IEEE, 2018, pp. 0210--0215.

\bibitem{angelov2020towards}
P.~Angelov, E.~Soares, {Towards explainable deep neural networks (xDNN)},
  Neural Networks 130 (2020) 185--194.

\bibitem{rieger2019aggregating}
L.~Rieger, L.~K. Hansen, {Aggregating explainability methods for neural
  networks stabilizes explanations}, arXiv preprint arXiv:1903.00519 (2019).

\bibitem{saadallah2021explainable}
A.~Saadallah, M.~Jakobs, K.~Morik, {Explainable online deep neural network
  selection using adaptive saliency maps for time series forecasting}, in:
  Joint European Conference on Machine Learning and Knowledge Discovery in
  Databases, Springer, 2021, pp. 404--420.

\bibitem{samek2016evaluating}
W.~Samek, A.~Binder, G.~Montavon, S.~Lapuschkin, K.-R. M{\"u}ller, Evaluating
  the visualization of what a deep neural network has learned, IEEE
  transactions on neural networks and learning systems 28~(11) (2016)
  2660--2673.

\bibitem{koh2020concept}
P.~W. Koh, T.~Nguyen, Y.~S. Tang, S.~Mussmann, E.~Pierson, B.~Kim, P.~Liang,
  {Concept bottleneck models}, in: International conference on machine
  learning, PMLR, 2020, pp. 5338--5348.

\bibitem{espinosa2022concept}
M.~Espinosa~Zarlenga, P.~Barbiero, G.~Ciravegna, G.~Marra, F.~Giannini,
  M.~Diligenti, Z.~Shams, F.~Precioso, S.~Melacci, A.~Weller, et~al., {Concept
  embedding models: Beyond the accuracy-explainability trade-off}, Advances in
  Neural Information Processing Systems 35 (2022) 21400--21413.

\bibitem{jia2021studying}
Y.~Jia, E.~Frank, B.~Pfahringer, A.~Bifet, N.~Lim, {Studying and exploiting the
  relationship between model accuracy and explanation quality}, in: Machine
  Learning and Knowledge Discovery in Databases. Research Track: European
  Conference, ECML PKDD 2021, Bilbao, Spain, September 13--17, 2021,
  Proceedings, Part II 21, Springer, 2021, pp. 699--714.

\bibitem{debbi2021causal}
H.~Debbi, {Causal explanation of convolutional neural networks}, in: Machine
  Learning and Knowledge Discovery in Databases. Research Track: European
  Conference, ECML PKDD 2021, Bilbao, Spain, September 13--17, 2021,
  Proceedings, Part II 21, Springer, 2021, pp. 633--649.

\bibitem{wang2023generalized}
P.~Wang, N.~Vasconcelos, A generalized explanation framework for visualization
  of deep learning model predictions, IEEE Transactions on Pattern Analysis and
  Machine Intelligence (2023).

\bibitem{slack2021reliable}
D.~Slack, A.~Hilgard, S.~Singh, H.~Lakkaraju, {Reliable post hoc explanations:
  Modeling uncertainty in explainability}, Advances in neural information
  processing systems 34 (2021) 9391--9404.

\bibitem{luo2022learning}
Y.~Luo, C.~Xu, Y.~Liu, W.~Liu, S.~Zheng, J.~Bian, {Learning differential
  operators for interpretable time series modeling}, in: Proceedings of the
  28th ACM SIGKDD Conference on Knowledge Discovery and Data Mining, 2022, pp.
  1192--1201.

\bibitem{raissi2019physics}
M.~Raissi, P.~Perdikaris, G.~E. Karniadakis, {Physics-informed neural networks:
  A deep learning framework for solving forward and inverse problems involving
  nonlinear partial differential equations}, Journal of Computational physics
  378 (2019) 686--707.

\bibitem{allamanis2017learning}
M.~Allamanis, P.~Chanthirasegaran, P.~Kohli, C.~Sutton, {Learning continuous
  semantic representations of symbolic expressions}, in: International
  Conference on Machine Learning, PMLR, 2017, pp. 80--88.

\bibitem{graves2014neural}
A.~Graves, G.~Wayne, I.~Danihelka, {Neural Turing Machines}, arXiv preprint
  arXiv:1410.5401 (2014).

\bibitem{kaiser2015neural}
{\L}.~Kaiser, I.~Sutskever, {Neural GPUs learn algorithms}, arXiv preprint
  arXiv:1511.08228 (2015).

\bibitem{trask2018neural}
A.~Trask, F.~Hill, S.~E. Reed, J.~Rae, C.~Dyer, P.~Blunsom, {Neural arithmetic
  logic units}, Advances in neural information processing systems 31 (2018).

\bibitem{loos2017deep}
S.~Loos, G.~Irving, C.~Szegedy, C.~Kaliszyk, {Deep network guided proof
  search}, arXiv preprint arXiv:1701.06972 (2017).

\bibitem{kuhnel2018latent}
L.~Kuhnel, T.~Fletcher, S.~Joshi, S.~Sommer, {Latent space non-linear
  statistics}, arXiv preprint arXiv:1805.07632 (2018).

\bibitem{shao2018riemannian}
H.~Shao, A.~Kumar, P.~Thomas~Fletcher, {The riemannian geometry of deep
  generative models}, in: Proceedings of the IEEE Conference on Computer Vision
  and Pattern Recognition Workshops, 2018, pp. 315--323.

\bibitem{ying2017manifold}
S.~Ying, Z.~Wen, J.~Shi, Y.~Peng, J.~Peng, H.~Qiao, {Manifold preserving: An
  intrinsic approach for semisupervised distance metric learning}, IEEE
  transactions on neural networks and learning systems 29~(7) (2017)
  2731--2742.

\bibitem{abdal2019image2stylegan}
R.~Abdal, Y.~Qin, P.~Wonka, {Image2stylegan: How to embed images into the
  stylegan latent space?}, in: Proceedings of the IEEE/CVF International
  Conference on Computer Vision, 2019, pp. 4432--4441.

\bibitem{almhaithawi2023construction}
D.~Almhaithawi, M.~Bertini, S.~Cuomo, F.~Panelli, A.~Bellini, T.~Cerquitelli,
  On the construction of numerical models through a prime convolutional
  approach, in: Proceedings of the 33rd European Safety and Reliability
  Conference (ESREL 2023), Research Publishing, Singapore, 2023, pp.
  2821--2829.

\bibitem{kolossvary2022distance}
I.~B. Kolossv{\'a}ry, I.~T. Kolossv{\'a}ry, {Distance between natural numbers
  based on their prime signature}, Journal of Number Theory 234 (2022)
  120--139.

\bibitem{ajit2020review}
A.~Ajit, K.~Acharya, A.~Samanta, {A review of convolutional neural networks},
  in: 2020 international conference on emerging trends in information
  technology and engineering (ic-ETITE), IEEE, 2020, pp. 1--5.

\bibitem{NEURIPS2019_9015}
A.~Paszke, S.~Gross, F.~Massa, A.~Lerer, J.~Bradbury, G.~Chanan, T.~Killeen,
  Z.~Lin, N.~Gimelshein, L.~Antiga, A.~Desmaison, A.~Kopf, E.~Yang, Z.~DeVito,
  M.~Raison, A.~Tejani, S.~Chilamkurthy, B.~Steiner, L.~Fang, J.~Bai,
  S.~Chintala,
  \href{http://papers.neurips.cc/paper/9015-pytorch-an-imperative-style-high-performance-deep-learning-library.pdf}{{PyTorch:
  An Imperative Style, High-Performance Deep Learning Library}}, in: Advances
  in Neural Information Processing Systems 32, Curran Associates, Inc., 2019,
  pp. 8024--8035.
\newline\urlprefix\url{http://papers.neurips.cc/paper/9015-pytorch-an-imperative-style-high-performance-deep-learning-library.pdf}

\bibitem{chen2021review}
L.~Chen, S.~Li, Q.~Bai, J.~Yang, S.~Jiang, Y.~Miao, {Review of image
  classification algorithms based on convolutional neural networks}, Remote
  Sensing 13~(22) (2021) 4712.

\bibitem{peng2020semicdnet}
D.~Peng, L.~Bruzzone, Y.~Zhang, H.~Guan, H.~Ding, X.~Huang, {SemiCDNet: A
  semisupervised convolutional neural network for change detection in high
  resolution remote-sensing images}, IEEE Transactions on Geoscience and Remote
  Sensing 59~(7) (2020) 5891--5906.

\end{thebibliography}


%\begin{thebibliography}{00}


%\bibitem{}

%\end{thebibliography}
\end{document}